%% file: plain.tex
\documentclass[11pt]{article} 
\usepackage{tabularx}
\usepackage{booktabs}
\usepackage[table]{xcolor}
\input{math_commands.tex}

\usepackage{natbib}

\usepackage[
            CJKbookmarks=true,
            bookmarksnumbered=true,
            bookmarksopen=true,
            colorlinks=true,
            citecolor=red,
            linkcolor=blue,
            anchorcolor=red,
            urlcolor=blue
            ]{hyperref}
\usepackage{url}
\usepackage{amsthm}
\usepackage{algorithm,algorithmic}

\newtheorem{definition}{Definition}
\newtheorem{assumption}{Assumption}
\newtheorem{theorem}{Theorem}
\newtheorem{lemma}{Lemma}
\newtheorem{remark}{Remark}

\newtheorem{corollary}{Corollary}

\newcommand{\wh}{\widehat}

\newcommand{\clip}{\operatorname{clip}}
\newcommand{\norm}[1]{\left\lVert #1\right\rVert}

\newcommand{\piref}{\pi^{\mathrm{ref}}}

\makeatletter

\title{Provable Benefits of Regularization: Fast Rates for Adversarial Imitation Learning}

\author{%
Hanbin Zhou$^{1,*}$ \qquad
Shangzhe Li$^{2,*}$ \qquad
Alexander Braverman$^{3}$ \qquad
Weitong Zhang$^{2,\dagger}$ \\[0.4em]
\small $^{1}$HKUST(GZ) \qquad
$^{2}$UNC Chapel Hill \qquad
$^{3}$Cornell University \\[0.2em]
\small $^{*}$Equal contribution. \quad
$^{\dagger}$Correspondence to: \texttt{weitongz@unc.edu}
}
\date{}

\begin{document}

\maketitle

\begin{abstract}
We study adversarial imitation learning (AIL), in which an agent learns to imitate expert demonstrations by optimizing a policy against an adversarial reward that distinguishes expert and learner behavior. Historically, reward regularization and entropy-based policy regularization are key components of empirically successful methods such as GAIL and LS-IQ, yet their finite-sample benefits remain underexplored. We establish fast rates for jointly regularized AIL in finite-horizon Markov decision processes with general function approximation. Our model-free algorithm, Dually Regularized AIL, combines KL policy regularization with a quadratic reward penalty weighted by expert and learner occupancies. With $K$ online episodes and $N$ expert trajectories, we prove a $\widetilde{\mathcal O}\left(\frac{1}{K}+\frac{1}{N}\right)$ bound on the regularized imitation gap for fixed regularization parameters. Our analysis combines an online mirror descent construction for general convex reward classes to control estimation error from finite expert data and stochastic learner feedback, with a sharp analysis of optimistic KL-regularized policy learning. To the best of our knowledge, Dually Regularized AIL is the first algorithm to simultaneously achieve $\widetilde{\mathcal O}\left(\frac{1}{\epsilon}\right)$ sample complexity in both expert demonstrations and online interactions for this regularized AIL objective, even with stochastic experts. These results provide a rigorous characterization of the complementary statistical benefits of reward and policy regularization in AIL.
\end{abstract}

\section{Introduction}
Imitation learning (IL) seeks to learn a policy that reproduces expert behavior from demonstrations, with the goal of matching the expert's performance when the task reward is unavailable. Behavior cloning approaches this goal by directly predicting expert actions. Adversarial imitation learning (AIL) instead learns a reward that distinguishes expert and learner behavior, and improves the policy under that reward~\citep{10.1145/1015330.1015430,NIPS2007_ca3ec598,ho2016gail}. This provides a performance-based criterion for imitation: expected cumulative reward depends on the state--action occupancy measure, so reducing the expert--learner gap over a reward class controls performance on the tasks represented by that class. Reward learning therefore guides the policy toward matching the expert's behavior across the trajectory.

As AIL has developed, reward and policy regularization have become important ingredients of practical algorithms. A common motivation for reward regularization is to reduce overfitting to finite demonstrations and maintain a useful learning signal for the policy~\citep{orsini2021matters}. Entropy-based policy regularization complements this by encouraging stochastic behavior and exploration during policy improvement~\citep{haarnoja2018soft}. GAIL incorporates both ingredients through cost-regularized occupancy matching and a causal-entropy policy objective~\citep{ho2016gail}. Related inverse-Q methods retain this two-sided structure: IQ-Learn combines implicit reward regularization with soft-$Q$ or soft actor--critic policy updates, while LS-IQ studies a quadratic reward penalty under a mixture of expert and learner occupancies within a maximum-entropy formulation~\citep{garg2021iqlearn,alhafez2023lsiq}. These developments motivate regularization as a way to limit overfitting and stabilize learning, and connect its design to the underlying occupancy-matching objective. How these benefits translate into faster convergence from finite demonstrations and online interactions remains less understood.

One way to make these statistical benefits precise is to examine whether regularization can improve convergence rates. Recent advances in IL and reinforcement learning (RL) show that structural assumptions and regularized objectives can yield faster rates than the conventional inverse-square-root dependence on sample size. For realizable deterministic experts, \citet{foster2024behavior} establish a $\widetilde{\mathcal O}(1/N)$ performance bound for log-loss behavior cloning from $N$ expert trajectories. In online RL, \citet{zhao2025kl} exploit the curvature of a KL-regularized policy objective to obtain logarithmic cumulative regret and a corresponding $\widetilde{\mathcal O}(1/K)$ average policy error after $K$ episodes. These results motivate examining whether reward regularization can similarly accelerate the convergence of AIL.

However, one challenging part in the AIL is that the coupling between the reward and the policy. In particular, since the reward is learned from the same finite expert dataset throughout training, while the evolving policy determines the distribution of newly collected learner trajectories. The reward learner must therefore control estimation errors from both data sources, and the policy learner must improve its behavior using an estimated value function. This coupling brings the two regularizers into the same statistical problem and raises the following question:
\begin{center}
\emph{Can reward and policy regularization jointly yield fast rates for AIL\\ in both expert demonstrations and online interactions?}
\end{center}

In this paper, we answer this question for jointly regularized AIL in finite-horizon MDPs with unknown transitions and general function approximation. We propose Dually Regularized AIL, a model-free algorithm combining KL policy regularization with a quadratic reward penalty weighted by expert and learner occupancies. The policy objective recovers causal-entropy regularization under a uniform reference. The algorithm alternates optimistic value-based policy learning with online mirror descent over a general convex reward class, reusing a fixed expert dataset and collecting one new learner trajectory per episode. We summarize our technical contributions as follows:

\vspace{.3em} \noindent \textbf{Fast rates for regularized AIL.}
With $K$ online episodes and $N$ expert trajectories, Theorem~\ref{thm:main} and Corollary~\ref{cor:sample_complexity} establish the high-probability bound $\operatorname{Gap}(\bar\pi_K,\bar r_K)
\leq \widetilde{\mathcal O}\left(\frac{1}{K}+\frac{1}{N}\right)$ on the regularized saddle-point gap, for fixed positive regularization parameters and controlled function-class complexity. This yields $\widetilde{\mathcal O}(\epsilon^{-1})$ sample requirements for both expert demonstrations and online interactions, even for stochastic experts. The full theorem specifies the dependence on the horizon, regularization strengths, covering numbers, and generalized eluder dimensions.

\vspace{.3em} \noindent \textbf{Statistical roles of regularization.}
Our reward analysis uses an online mirror descent construction whose update stability is controlled by generalized eluder dimension. The occupancy-weighted quadratic penalty supplies curvature that controls estimation error from the reused expert dataset and stochastic feedback from learner trajectories. We combine this construction with the sharp optimistic KL-RL analysis of~\citet{zhao2025kl}, in which policy curvature converts planning uncertainty into squared Bellman errors. Together, these arguments explain how reward and policy regularization contribute to fast finite-sample rates in AIL.

\section{Related Work}
In this section, we provide an overview of the related works in AIL and RL settings. Table~\ref{tab:regularization_rates} summarizes representative guarantees, separating their objectives and expert assumptions.

\vspace{.3em} \noindent \textbf{Reward regularization and inverse Q-learning.}
Apprenticeship learning connects reward learning to expert-performance matching~\citep{10.1145/1015330.1015430,NIPS2007_ca3ec598}. GAIL formalizes the connection between a convex cost regularizer and the induced occupancy discrepancy through convex conjugacy; its particular choice gives a Jensen--Shannon objective together with causal-entropy regularization~\citep{ho2016gail}. IQ-Learn represents reward and policy through a single soft Q-function, avoiding explicit alternation between a separately parameterized reward and policy~\citep{garg2021iqlearn} which can be considered as the max-entropy RL for policy regularization. Building on this perspective, LS-IQ studies a squared reward penalty under the expert--learner mixture and relates it to a bounded Pearson $\chi^2$ divergence, reward bounds, and improved stability~\citep{alhafez2023lsiq}. Empirical comparisons also show that discriminator regularization can matter substantially, especially on harder tasks, although its effect depends on other algorithmic choices~\citep{orsini2021matters}. These works motivate our regularizer but address a different question from the joint finite-sample rates proved here. In particular, a quadratic penalty controls reward magnitude and supplies curvature; it need not enforce smoothness with respect to state--action inputs.

\begin{table}[t]
\centering
\caption{Sample complexity in expert trajectories ($N$) and online episodes ($K$) in related works. Sample complexity omits the horizon, reward-range, and structural-complexity factors. ``General'' experts may be stochastic. All guarantees are high-probability except Mimic-Emp's expected gap. Unregularized rates summarize the worst case, omitting variance-dependent improvements. These works made regular RL / IL assumptions where we refer readers to check their original statements.}
\label{tab:regularization_rates}
\begingroup
\small
\setlength{\tabcolsep}{4pt}
\renewcommand{\arraystretch}{1.20}

\begin{tabularx}{\linewidth}{@{}>{\raggedright\arraybackslash}Xllllc@{}}
\toprule
Method & Class & Expert & Criterion & $N$ & $K$ \\
\midrule
Mimic-Emp~\citep{rajaraman2020toward}
& Tabular & General & Exp. IL gap & $\epsilon^{-1}$ & $0$ \\

Log-loss BC~\citep{foster2024behavior}
& General & Det. & IL gap & $\epsilon^{-1}$ & $0$ \\

Log-loss BC~\citep{foster2024behavior}
& General & General & IL gap & $\epsilon^{-2}$ & $0$ \\

Min-Max-IRL~\citep{schlaginhaufen2026fast}
& Linear & General & KL risk & $\epsilon^{-1}$ & $0$ \\

MB-TAIL~\citep{xu2023provably}
& Tabular & Det. & IL gap & $\epsilon^{-1}$ & $\epsilon^{-2}$ \\

OPT-AIL~\citep{xu2024optail}
& General & General & IL gap & $\epsilon^{-2}$ & $\epsilon^{-2}$ \\

MB-AIL~\citep{li2026mbail}
& General & General & IL gap & $\epsilon^{-2}$ & $\epsilon^{-2}$ \\
\midrule

KL-LSVI-UCB~\citep{zhao2025kl}
& General & -- & Reg. RL regret & -- & $\epsilon^{-1}$ \\
\midrule

\rowcolor[gray]{0.94}
\textbf{Dually Regularized AIL (ours)}
& \textbf{General}
& \textbf{General}
& \textbf{Reg. dual gap}
& $\boldsymbol{\epsilon^{-1}}$
& $\boldsymbol{\epsilon^{-1}}$ \\

\bottomrule
\end{tabularx}

\endgroup
\end{table}

\vspace{.3em} \noindent \textbf{Finite-sample adversarial imitation learning.}
Online apprenticeship learning combines optimistic exploration with no-regret reward and policy updates, yielding square-root interaction regret and an additional error from finite expert data~\citep{shani2022online}. Subsequent analyses cover unknown tabular dynamics~\citep{xu2023provably}, linear-mixture models~\citep{liu2021provably}, and linear MDPs~\citep{vianoimitation}. Under general function approximation, OPT-AIL couples reward optimization with optimistic policy learning and characterizes complexity using a generalized eluder coefficient~\citep{xu2024optail}. MB-AIL instead learns a transition model and establishes second-order guarantees that adapt to stochasticity, together with lower bounds on expert and interaction requirements~\citep{li2026mbail}. Its worst-case square-root terms can improve in favorable low-variance instances. Our result studies a regularized saddle-point criterion and uses reward curvature to obtain inverse-sample dependence without requiring a deterministic expert or vanishing variance. The different objectives and structural assumptions are essential to interpreting the comparison.

\paragraph{Fast rates and the role of expert stochasticity.}
Fast expert-sample rates depend on both the policy class and the expert's stochasticity. In tabular MDPs, \citet{rajaraman2020toward} obtain inverse-$N$ expected-error bounds even for stochastic experts. For general policy classes, the worst-case picture is different. For a realizable deterministic expert, \citet{foster2024behavior} combine log-loss estimation with a trajectory-level analysis to obtain $\widetilde{\mathcal O}(N^{-1})$ performance error; stochastic experts admit variance-sensitive bounds with a worst-case $N^{-1/2}$ term. Their deterministic-expert result does not require deterministic transitions. Fast expert-sample rates also arise in model-based tabular AIL with deterministic experts~\citep{xu2023provably}, while the second-order AIL guarantees of~\citet{li2026mbail} quantify more general instance-dependent improvements. In entropy-regularized min-max IRL, \citet{schlaginhaufen2026fast}
establish a $\widetilde{\mathcal O}(N^{-1})$ rate for excess
trajectory-level KL risk with linear reward classes, allowing
stochastic experts and model misspecification. Our fast rate arises from the regularized objective rather than expert determinism. It therefore complements these results instead of contradicting lower bounds for ordinary imitation. In particular, the success of offline cloning rules out a universal claim that online interaction is necessary; our result characterizes the efficiency of a reward-based alternative.

\section{Preliminaries}
We consider an episodic Markov decision process (MDP) $\langle\mathcal{S},\mathcal{A},H,P\rangle$, where $\mathcal{S}$ and $\mathcal{A}$ are the state space and action space, $H\in\mathbb{Z}_+$ is the horizon, and $P=\{P_h\}_{h\in[H]}$ with $P_h:\mathcal{S}\times\mathcal{A}\rightarrow\Delta_{\mathcal{S}}$ is the transition dynamics. We assume the initial state is fixed from $s_1$. A policy is denoted by $\pi=\{\pi_h\}_{h\in[H]}$, where $\pi_h:\mathcal{S}\rightarrow\Delta_{\mathcal{A}}$. For each $h\in[H]$, let $d_h^\pi$ denote the state-action occupancy distribution of a state-action pair $(s_h,a_h)$ induced by $P$ and
$\pi$ at timestep $h$.
We write $\langle\mu,f\rangle:=\int f \mathrm d\mu$ for a signed measure $\mu$ and a bounded function $f$ as the expectation $\mathbb E_{\mu}[f(\cdot)]$ for simplicity.


In \textbf{imitation learning}, consider $\pi^E$ be the expert policy and denote its occupancy by $d_h^E$. We observe an expert dataset $\mathcal{D}^E=\{s_h^i,a_h^i\}_{h\in[H]}^{i\in[N]}$ containing $N$ independent expert trajectories. The agent then performs $K$-round of online interactions with the environments. In the AIL setting, the agent seeks a stage-wise reward function $r=\{r_h\}_{h\in[H]}$ in order to distinguish the expert demonstrations and the learner behavior. We assume the reward function in this class bounded by $r(s, a) \in [-1, 1]$.


\paragraph{KL regularized policy updates.}
Let $\piref$ be a reference policy that is known
to the learner and let $\tau>0$ be the KL-regularization coefficient.
The reference policy $\piref$ may be initialized from a policy
pretrained via behavioral cloning, similar to prior empirical
approaches~\citep{jena2021augmenting}. Alternatively, choosing $\piref$
to be uniform reduces KL regularization to entropy regularization
up to an additive constant, recovering the maximum-entropy AIL
formulation commonly used in prior work, including
GAIL~\citep{ho2016gail} and IQ-Learn~\citep{garg2021iqlearn}.
We use the convention $0\log(0/q)=0$, with KL divergence equal to
$+\infty$ when its first argument assigns positive mass outside the
support of its second. Gibbs identities are understood on the support
of $\piref$.
Following prior work~\citep{zhao2025kl}, define the KL-regularized
value and action-value functions for a policy $\pi$ under any bounded
stagewise reward $r$ as
\begin{align}
\notag 
V_{r,h}^{\pi,\tau}(s)
&:=\mathbb{E}_{P,\pi}\left[
\textstyle{\sum_{t=h}^H}\left(
r_t(s_t,a_t)
-\tau\log{\pi_t(a_t\mid s_t)} / {\piref_t(a_t\mid s_t)}
\right)\,\middle|\,s_h=s\right], \\
\notag 
    Q_{r,h}^{\pi,\tau}(s,a)
    &:=r_h(s,a)+[P_hV_{r,h+1}^{\pi,\tau}](s,a),
\end{align}
where $V_{r,H+1}^{\pi,\tau}\equiv0$ and
$[P_hV](s,a):=\mathbb{E}_{s'\sim P_h(\cdot\mid s,a)}[V(s')]$.
Equivalently,
\begin{align}
\notag 
    V_{r,h}^{\pi,\tau}(s)
    =\mathbb{E}_{a\sim\pi_h(\cdot\mid s)}
    \left[Q_{r,h}^{\pi,\tau}(s,a)\right]
    -\tau\operatorname{KL}\!\left(
    \pi_h(\cdot\mid s)\,\|\,\piref_h(\cdot\mid s)
    \right).
\end{align}
\paragraph{Regularized reward update}
While KL-divergence to a reference policy characterizes the policy-side regularization, for a fixed reward, our setup focuses on adversarial imitation learning (AIL), in which the objective compares the learner with the expert uniformly over the reward class $\mathcal{R}$. We complement the policy-side KL regularization with a quadratic reward regularizer, following prior practical AIL methods such as LS-IQ~\citep{alhafez2023lsiq}. For an expert occupancy weight $\omega\in(0,1)$, define the quadratic reward regularizer
\begin{align}
\notag 
    \psi_\pi^\omega(r)
    :=\frac{\alpha}{2}\sum_{h=1}^H
    \left[\omega\mathbb E_{d_h^E}[r_h^2]
    +(1-\omega)\mathbb E_{d_h^\pi}[r_h^2]\right],
    \qquad \alpha>0.
\end{align}
For any reward $r$, write
$J_r^\tau(\pi):=V_{r,1}^{\pi,\tau}(s_1)$.
The regularized AIL objective can be written by the following min-max optimization on policy $\pi$ and reward $r$:
\begin{equation}
\textstyle{\max_{r \in \mathcal R} \min_{\pi}} \mathcal L(\pi,r)
:=J_r^\tau(\pi^E)-J_r^\tau(\pi)-\psi_\pi^\omega(r)
\label{eqn:ail_payoff}
\end{equation}
We then define the regret over $K$ online episodes by the cumulative dual gap as
\begin{align}
\label{eqn:regret}
\operatorname{Regret}_{\omega,\alpha,\tau}(K)
\textstyle{:=\sup_{r\in\mathcal R}\sum_{k=1}^K\mathcal L(\pi_k,r)
-\inf_\pi\sum_{k=1}^K\mathcal L(\pi,r_k).}
\end{align}




\paragraph{AIL with Function Approximation.}
We consider model-free adversarial imitation learning under general
function approximation. Specifically, let
$\mathcal F=\{\mathcal F_h\}_{h\in[H]}$, where each
$\mathcal F_h$ is a nonempty class of
functions from $\mathcal S\times\mathcal A$ to
$[-2V_{\max},2V_{\max}]$ for $V_{\max}:=H\left(1+\frac{\alpha(1-\omega)}2\right)$. The algorithm fits its regression
functions in these classes. To characterize the exploration
complexity of the underlying MDP, we adopt the generalized Eluder
dimension introduced by~\citet{agarwal2023vo}:

\begin{definition}[Generalized eluder dimension]
\label{def:generalized-eluder}
Let $\lambda>0$, $h\in[H]$, and $Z_h=(z_{k,h})_{k=1}^K$, where $z_{k,h}=(s_{k,h},a_{k,h})$, be an indexed sequence of state--action pairs, with repetitions retained. Write $z_{[k-1],h}:=(z_{j,h})_{j=1}^{k-1}$. For a bounded class $\mathcal F_h$ of real-valued functions on $\mathcal S\times\mathcal A$, define
\begin{align*}
\dim_K(\mathcal F_h,\lambda)
&:={\textstyle\sup_{|Z_h|=K} \sum_{k=1}^K}
\min\left\{
1,D^2_{\mathcal F_h}(z_{k,h};z_{[k-1],h};\lambda)
\right\},\\
D^2_{\mathcal F_h}(z_{k,h};z_{[k-1],h};\lambda)
&:=\sup_{f_1,f_2\in\mathcal F_h}
\frac{
\left(f_1(z_{k,h})-f_2(z_{k,h})\right)^2
}{
\sum_{j\in[k-1]}
\left(f_1(z_{j,h})-f_2(z_{j,h})\right)^2+\lambda
}.
\end{align*}
We write $D_{\mathcal F_h}:=\sqrt{D^2_{\mathcal F_h}}$ and denote
$\dim_K(\mathcal F,\lambda)
=H^{-1}\sum_{h\in[H]}\dim_K(\mathcal F_h,\lambda)$ for simplicity.
These definitions also apply to $\mathcal R_h$, with
$\dim_K(\mathcal R,\lambda):=
H^{-1}\sum_{h=1}^H\dim_K(\mathcal R_h,\lambda)$.
\end{definition}

\begin{remark}
The generalized Eluder dimension in Definition~\ref{def:generalized-eluder} corresponds to the unweighted case $\sigma\equiv1$ of~\citet{agarwal2023vo}. It is controlled by the standard Eluder dimension
$\dim_E(\mathcal{F},\epsilon)$~\citep{NIPS2013_41bfd20a}, up to logarithmic factors and regularization-dependent terms~\citep{zhao2024nearly}. 
\end{remark}

In addition to sequential uncertainty, we quantify function-class size through covering numbers, which enter the uniform concentration bounds in our analysis.

\begin{definition}[Covering numbers]
\label{def:covering}
For $\kappa>0$ and a function class $\mathcal C$, let $\mathcal N_{\mathcal C}(\kappa)$ denote the smallest cardinality of a $\kappa$-cover $\mathcal C_\kappa\subseteq\mathcal C$: for every $f\in\mathcal C$, there exists $f'\in\mathcal C_\kappa$ such that $\norm{f-f'}_\infty\leq\kappa$.
\end{definition}



\section{Dually Regularized Adversarial Imitation Learning}

\begin{algorithm}[t]
\caption{Dually Regularized Adversarial Imitation Learning}
\label{alg:main}
\begin{algorithmic}[1]

\STATE \textbf{Input:} Expert dataset $\mathcal{D}^E$,
convex reward class $\mathcal{R}$,
$Q$-function class $\mathcal{F}=\{\mathcal{F}_h\}_{h\in[H]}$,
reference policy $\piref$,
$\alpha,\tau>0$, mixture weight $\omega\in(0,1)$, OMD parameter $\rho\in(0,1)$, regularization parameter $\lambda>0$,
confidence radius $\beta>0$, and number of episodes $K$.

\STATE Choose $r_1\in\mathcal{R}$.

\FOR{$k=1,\ldots,K$}

    \STATE For each $h\in[H]$, define
    $\displaystyle
    u_{k,h}:=u_{r_k,h},
    $
    and set $\wh V_{k,H+1}\equiv 0$.

    \FOR{$h=H,H-1,\ldots,1$}

        \STATE\label{line:regression}
        $\displaystyle
        \wh f_{k,h}
        \in
        \arg\min_{f\in\mathcal{F}_h}
        \sum_{j=1}^{k-1}
        \left(
            f(s_{j,h},a_{j,h})
            -
            u_{k,h}(s_{j,h},a_{j,h})
            -
            \wh V_{k,h+1}(s_{j,h+1})
        \right)^2.
        $

        \STATE
        $\displaystyle
        b_{k,h}(s,a)
        =
        \min\left\{
            4V_{\max},
            \,
            \beta\,
            D_{\mathcal{F}_h}
            \left(
                (s,a);
                \mathcal D_{k-1,h};\lambda
            \right)
        \right\}.
        $

        \STATE
        $\displaystyle
        \wh Q_{k,h}(s,a)
        =
        \clip_{[-V_{\max},V_{\max}]}
        \left(
            \wh f_{k,h}(s,a)
            +
            b_{k,h}(s,a)
        \right).
        $

        \STATE
        $\displaystyle
        \wh V_{k,h}(s)
        =
        \tau
        \log
        \mathbb{E}_{a\sim\piref_h(\cdot\mid s)}
        \left[
            \exp\left(
                \frac{\wh Q_{k,h}(s,a)}{\tau}
            \right)
        \right].
        $

        \STATE\label{line:policy-update}
        $\displaystyle
        \pi_{k,h}(a\mid s)
        =
        \piref_h(a\mid s)
        \exp\left(
            \frac{
                \wh Q_{k,h}(s,a)
                -
                \wh V_{k,h}(s)
            }{\tau}
        \right).
        $

    \ENDFOR

    \STATE\label{line:policy-execute} Execute $\pi_k$ for one episode and observe
    $\{(s_{k,h},a_{k,h},s_{k,h+1})\}_{h=1}^H$.

    \STATE\label{line:reward-update} Form the observed reward gradient in
    Eq.~\ref{eq:observed-gradient} and compute
    $r_{k+1}$ according to Eq.~\ref{eq:reward-update}.

\ENDFOR

\STATE \textbf{return}
$\displaystyle
\bar{\pi}_K
=
\mathrm{Unif}\{\pi_1,\ldots,\pi_K\},
$
where one policy is sampled uniformly at the beginning of each
evaluation episode.

\end{algorithmic}
\end{algorithm}

In this section, we introduce Dually Regularized AIL, summarized in Algorithm~\ref{alg:main}. The algorithm takes as input an expert dataset $\mathcal{D}^E$ consisting of $N$ length-$H$ trajectories generated by the expert policy $\pi^E$, together with a reference policy $\piref$. We consider two general function classes: $\mathcal{F}$ for approximating the action-value functions and $\mathcal{R}$ for learning the adversarial reward.

Dually Regularized AIL proceeds for $K$ episodes, each consisting of two phases. In the first phase, analogous to the standard KL-regularized RL \citep{zhao2025kl}, the algorithm performs backward least-squares value estimation on a regularized reward, constructs exploration bonuses from the resulting confidence bounds, and updates the policy using the closed-form solution to the KL-regularized policy optimization problem (Lines~\ref{line:regression}--\ref{line:policy-update}). The resulting policy is then executed in the environment to collect a new trajectory (Line~\ref{line:policy-execute}). In the second phase, the adversarial reward is updated using online mirror descent based on the expert and the behavioral episode from current policy (Line~\ref{line:reward-update}). We describe these two phases in detail below.

\vspace{.3em}\noindent \textbf{Phase I: Backward Planning and Least-Squares Regression.}
For fixed $r_k$, we incorporate the reward regularization in $\psi_{\pi_k}^{\omega}$ with the shaped reward by $ u_{r,h}(s,a):=r_h(s,a) +\tfrac12 \cdot {\alpha(1-\omega)} r_h(s,a)^2.$ Notably, by the definition of the regularized AIL objective Eq.~\ref{eqn:ail_payoff}, we have 
\[
\mathcal L(\pi,r)\textstyle 
=J_{u_r}^\tau(\pi^E)-J_{u_r}^\tau(\pi)
-\tfrac12 \cdot \alpha\sum_{h=1}^H\mathbb E_{d_h^E}[r_h^2].
\]
Thus, for fixed $r$, minimizing $\mathcal L(\pi,r)$ over $\pi$
is equivalent to maximizing $J_{u_r}^\tau(\pi)$. 



Then starting from $\wh V_{k,H+1}\equiv0$, we estimate the value function by fitting the Bellman targets over $\mathcal{F}_h$ using all previously collected transitions:
\begin{align}
\label{eq:planning-regression}
\wh f_{k,h}\textstyle{
\in\arg\min_{f\in\mathcal{F}_h}
\sum_{j=1}^{k-1}
\left(
f(s_{j,h},a_{j,h})
-u_{k,h}(s_{j,h},a_{j,h})
-\wh V_{k,h+1}(s_{j,h+1})
\right)^2.}
\end{align}
The targets are recomputed using the current shaped reward and
continuation value, allowing historical transitions to support
planning under changing rewards and policies without estimating
the transition dynamics. We then add the exploration bonus
\[
b_{k,h}(s,a):=\min\left\{4V_{\max},\,
\beta D_{\mathcal F_h}((s,a);\mathcal D_{k-1,h};\lambda)\right\},
\]
where $\mathcal D_{k-1,h}:=\{(s_{j,h},a_{j,h})\}_{j=1}^{k-1}$
denotes the history of trajectories at step $h$ before the episode $k$.





The algorithm clips the resulting action-value estimates.
On the confidence event
proved in Lemma~\ref{lem:app-regression-confidence}, these estimates are
optimistic and encourage exploration of uncertain state--action pairs. Finally, the soft Bellman update yields $\wh V_{k,h}$ and the
closed-form policy $\pi_{k,h}$, balancing optimistic action values
against KL deviation from $\piref$. We execute $\pi_k$ for one
episode and retain the trajectory for the reward update in Phase II
and subsequent planning.

\paragraph{Phase II: Adversarial Reward Learning.}
For fixed $\pi_k$, the reward learner therefore seeks to control regret with respect to the following population loss
\begin{align}
\ell_k(r)\textstyle{
:=\sum_{h=1}^H\left\{
\mathbb E_{d_h^{\pi_k}}\!\left[
r_h+\tfrac12 \cdot {\alpha(1-\omega)}r_h^2\right]
-\mathbb E_{d_h^E}\!\left[
r_h-\tfrac12 \cdot {\alpha\omega}r_h^2\right]\right\}.}
\label{eq:population-reward-loss}
\end{align}
By the unknown nature of the environment and expert policy,
the reward learner, however, is only exposed to the following empirical loss.
After observing episode $k$, let
\(
\hat d_{k,h}:=\delta_{(s_{k,h},a_{k,h})},
\hat d_h^E:=\frac1N\sum_{i=1}^N\delta_{(s_h^i,a_h^i)},
\)
where $\delta_{(s,a)}\in\mathbb R^{\mathcal S\times\mathcal A}$ denotes the one-hot vector at $(s,a)$.
The empirical loss is
\begin{align}
\wh\ell_k(r)\textstyle{
:=\sum_{h=1}^H\left\{
\mathbb E_{\hat d_{k,h}}\!\left[
r_h+\tfrac12 \cdot {\alpha(1-\omega)} r_h^2\right]
-\mathbb E_{\hat d_h^E}\!\left[
r_h-\tfrac12 \cdot {\alpha\omega} r_h^2\right]\right\}.}
\label{eq:empirical-reward-loss}
\end{align}
The first expectation is evaluation at $(s_{k,h},a_{k,h})$;
the second is the average over the $N$ expert trajectories.
$\ell_k$ is used in the analysis, while $\wh{\ell}_k$ determines the reward update.

After observing the loss feedback, we update the reward using the following online mirror descent (OMD) update rule with parameter $\rho\in(0,1)$:
\begin{equation}
\label{eq:reward-update}\textstyle
r_{k+1}\in\arg\min_{r\in\mathcal R}
\left\{\langle\wh g_k,r\rangle
+\tfrac12 \cdot {\alpha\rho} \sum_{h=1}^H\norm{r_h-r_{k,h}}_{k,h}^2
\right\}.
\end{equation}
where $\wh g_k:=(\wh g_{k,h})_{h=1}^H$ is the observed loss gradient vector at $r_k$, with components
\begin{align}
\wh g_{k,h}
:=\left(1+\alpha(1-\omega)r_{k,h}\right)\circ\hat d_{k,h}
-\left(1-\alpha\omega r_{k,h}\right)\circ\hat d_h^E\,,
\label{eq:observed-gradient}
\end{align}
for $\circ$ denoting the elementwise multiplication, and $\norm{r_h-r_{k,h}}_{k,h}^2$ is the Bregman divergence between $r_h$ and $r_{k,h}$ with respect to the following squared seminorm function
\begin{align}
\norm{r}_{k,h}^2
:=\frac{k\omega}{N}\sum_{i=1}^N r(s_h^i,a_h^i)^2
+(1-\omega)\sum_{j=1}^k r(s_{j,h},a_{j,h})^2,
\quad\forall r\in \mathbb{R}^{\mathcal{S}\times\mathcal{A}}.
\label{eq:reward-seminorm}
\end{align}

We note that no-regret learning algorithms such as follow-the-regularized-leader (FTRL) and online mirror descent (OMD)~\citep{hazan2023oco} have also been used for policy learning~\citep{sun2019obervation}, reward learning~\citep{li2026mbail,xu2024optail}, or both~\citep{shani2022online} in many other AIL works.
Our approach specializes this approach with an adaptive proximal penalty choice,
which aligns with our choice of reward penalty with only an extra multiplier $\rho$.
This design allows us to derive a reward class Eluder dimension bound for a stability term derived from the OMD analysis and offset part of the sampling error.
These two benefits together help us to provide a sharp bound of the reward error in later analysis.
This choice also falls within the general framework of adaptive mirror descent discussed in \citet{mcmahan2017adaptive}.
\section{Theoretical Analysis}

{\color{purple}
}
In this section, we present the sharp regret and sample complexity guarantee for Algorithm~\ref{alg:main}. We start by several common assumptions related to the general function approximation:

\begin{definition}[Bonus classes]
\label{def:bonus-class}
For each $h\in[H]$, let $\mathcal B_h$ be a nonempty deterministic
class of functions from $\mathcal S\times\mathcal A$
to $[0,4V_{\max}]$ containing every bonus $b_{k,h}(s,a) = \min\left\{ 4V_{\max}, \beta D_{\mathcal{F}_h} \left((s,a); \mathcal D_{k-1,h};\lambda\right)\right\}$ generated by
Algorithm~\ref{alg:main}, over all episodes and possible data
histories at the prescribed deterministic planner parameters.
\end{definition}

We further assume the reward class $\mathcal R$, value function class $\mathcal F$ and its bonus class $\mathcal B$ have finite covering numbers $\overline{\mathcal N}_{\mathcal R}(\kappa)$, $\overline{\mathcal N}_{\mathcal F}(\kappa)$, $\overline{\mathcal N}_{\mathcal B}(\kappa)$ under any $\kappa >0$. then we assume the Bellman completeness of the value estimation for any reward under general function approximation which is a common assumption for the theoretical analysis in model-free RL \citep{agarwal2023vo,zhao2024nearly}.

\begin{assumption}[Bellman completeness]
\label{ass:optimistic_planner}
For every $h\in[H]$, $r\in\mathcal R$, and $V\in\mathcal V_{h+1}$
from the optimistic soft-value class in
Definition~\ref{def:planner-value-class},
$u_{r,h}+P_hV\in\mathcal F_h$.
The terminal continuation is $V\equiv0$.
\end{assumption}
With these assumption for completeness, we are ready to present our theorem.
\begin{theorem}
\label{thm:main}
Under the stated assumptions, for any $\delta\in(0,1)$ and $\omega\in(0,1)$, set parameter as $\rho=1/2$, $\lambda>0$, and $\beta:=\sqrt{\lambda+
64V_{\max}^2\log\frac{
8HK\overline{\mathcal N}_{\mathcal F\oplus\mathcal B}(V_{\max}/K^2)}{\delta}}$, with probability at least $1 - \delta$,
\begin{align}
\operatorname{Regret}_{\omega,\alpha,\tau}(K)
&\leq\mathcal O\!\bigg(
\frac{H(1+\alpha\omega)^2}{\alpha\omega}(1+\log K+\frac{K}{N}\log\frac{4H\overline{\mathcal N}_{\mathcal R}(1/N^2)^2}{\delta})
\nonumber\\
&\qquad+
\frac{H(1+\alpha(1-\omega))^2}{\alpha(1-\omega)}
(\dim_K(\mathcal R,4/K^2)+\log\frac{4H\overline{\mathcal N}_{\mathcal R}(1/K^2)}{\delta})
\nonumber\\
&\qquad+
\frac{H^3\beta^2}{\tau}
\left[\dim_K(\mathcal F,\lambda)+\log\frac{8H}{\delta}\right]
\bigg).
\label{eqn:regret_bound}
\end{align}

where
$\overline{\mathcal N}_{\mathcal F\oplus\mathcal B}(\kappa)
:=\max_{h\in[H]}\left\{
\mathcal N_{\mathcal F_h}(\kappa)
\mathcal N_{\mathcal F_{h+1}}(\frac{\kappa}2)
\mathcal N_{\mathcal B_{h+1}}(\frac{\kappa}2)\right\}$, $\overline{\mathcal N}_{\mathcal R}(\kappa):=\max_{h\in[H]}
\mathcal N_{\mathcal R_h}(\kappa)$
\end{theorem}
Theorem~\ref{thm:main} immediately yields the following
sample-complexity guarantee.


\begin{corollary}[Sample complexity of Dually Regularized AIL]
\label{cor:sample_complexity}
Under the conditions of Theorem~\ref{thm:main}, suppose
$\lambda\leq V_{\max}^2$, and write
$d_{\mathcal R}:=\dim_K(\mathcal R,4/K^2)$ and
$d_{\mathcal F}:=\max\{1,\dim_K(\mathcal F,\lambda)\}$.
For any $\epsilon>0$, the returned policy mixture
$\bar\pi_K=\operatorname{Unif}\{\pi_1,\ldots,\pi_K\}$ and
averaged reward $\bar r_K=K^{-1}\sum_{k=1}^K r_k$ satisfy,
with probability at least $1-\delta$, the averaged dual gap satisfied
\[\textstyle
0\leq \sup_{r \in \mathcal R} \mathcal L(\bar \pi_K, r) - \inf_{\pi} \mathcal L(\pi, \bar r_K) \leq \epsilon
\]
with sufficient interaction and expert sample requirements
\begin{align*}
K
&=\widetilde{\mathcal O}\!\Bigg(
\frac{H}{\epsilon}\Bigg[
\frac{(1+\alpha\omega)^2}{\alpha\omega}
+\frac{(1+\alpha(1-\omega))^2}{\alpha(1-\omega)}
\left(
d_{\mathcal R}
+\log\!\left(2\,\overline{\mathcal N}_{\mathcal R}(1/K^2)\right)
\right)
\Bigg]\\
&\hspace{65pt}
+\frac{H^3V_{\max}^2d_{\mathcal F}}{\tau\epsilon}
\log\!\left(
2\,\overline{\mathcal N}_{\mathcal F\oplus\mathcal B}
(V_{\max}/K^2)
\right)
\Bigg),\\[3pt]
N
&=\widetilde{\mathcal O}\!\left(
\frac{H(1+\alpha\omega)^2}{\alpha\omega\epsilon}
\log\!\left(2\,\overline{\mathcal N}_{\mathcal R}(1/N^2)\right)
\right).
\end{align*}
These are implicit sufficient conditions because the complexities
and covering scales can depend on $K$ and $N$. Sample counts are
rounded up to positive integers. The notation
$\widetilde{\mathcal O}$ suppresses additional logarithmic factors
in $H,K$, and $1/\delta$.
\end{corollary}
\begin{remark}[Interpretation of the sample complexity]
For fixed $\alpha,\tau>0$ and $\omega\in(0,1)$, we have $V_{\max}=\Theta(H)$. Corollary~\ref{cor:sample_complexity} therefore yields a $\widetilde{\mathcal O}\left(\tau^{-1}H^5 d_{\mathcal F}\log\bigl(\overline{\mathcal N}_{\mathcal F\oplus\mathcal B}(V_{\max}/K^2)\bigr)/\epsilon\right)$ policy-planning contribution to the required number of online interactions, while the expert-trajectory requirement is $\mathcal O\left(H\log\bigl(\overline{\mathcal N}_{\mathcal R}(1/N^2)/\delta\bigr)/\epsilon\right)$. For comparison, after translating the KL-LSVI-UCB result of~\citet{zhao2025kl} to our notation and horizon convention, its interaction complexity is $\widetilde{\mathcal O}\left(\tau^{-1}H^5 d_{\mathcal F}\log\bigl(\overline{\mathcal N}_{\mathcal F\oplus\mathcal B}(V_{\max}/K^2)\bigr)/\epsilon\right)$. Thus, the policy-planning component of Algorithm~\ref{alg:main} matches the interaction-complexity order of KL-LSVI-UCB.

\end{remark}
\begin{remark}
    Regarding the dependence on the precision $\epsilon$, our algorithm achieves an expert sample complexity of $\widetilde{\mathcal O}(\epsilon^{-1})$ for both deterministic and stochastic experts, with fixed regularization parameters and controlled function-class complexity. In contrast, for the unregularized imitation gap under general policy classes, \citet{foster2024behavior} establish worst-case upper bounds of $\widetilde{\mathcal O}(\epsilon^{-1})$ for deterministic experts and $\widetilde{\mathcal O}(\epsilon^{-2})$ for stochastic experts, together with matching $\Omega(\epsilon^{-1})$ and $\Omega(\epsilon^{-2})$ lower bounds in their dependence on $\epsilon$. Similarly, \citet{li2026mbail} derive an interaction lower bound for unregularized AIL containing an $\epsilon^{-2}\exp(-N)$ term, where $N$ is the number of expert trajectories. Our faster expert and interaction rates do not contradict these results: our guarantee controls the regularized saddle-point gap for fixed positive reward and policy regularization, rather than the unregularized imitation gap considered in those works. Consequently, these rates do not directly imply the same sample complexity for the unregularized objective, for which regularization bias and the dependence on the regularization parameters must also be accounted for.
\end{remark}

\begin{remark}[Role of reward regularization]
The quadratic regularizer $\psi_\pi^\omega$ controls both reward-update stability and sampling error. Matching the proximal penalty to the empirical curvature yields an OMD stability bound governed by the generalized eluder dimension of $\mathcal R$, while $\rho\in(0,1)$ reserves curvature to absorb both expert and learner sampling errors.

The weight $\omega$ allocates curvature between expert and learner occupancies; approaching either endpoint weakens the corresponding error bound. The symmetric choice $\omega=1/2$ matches the equal-mixture regularizer in LS-IQ~\citep{alhafez2023lsiq}, which admits a bounded Pearson $\chi^2$-divergence interpretation. Our analysis complements this perspective with finite-sample guarantees: together with KL-regularized planning, any fixed $\omega\in(0,1)$ preserves the $\widetilde{\mathcal O}(1/K+1/N)$ regularized gap bound for fixed regularization parameters and controlled function-class complexity.
\end{remark}

\section{Proof Sketch}
We sketch the proof of Theorem~\ref{thm:main}, emphasizing how reward curvature controls both optimization and estimation, while policy curvature controls planning uncertainty. We suppress dependence on the fixed reward parameters $\alpha,\omega,\rho$ and additional logarithmic factors; Eq.~\ref{eqn:regret_bound} and Appendix~\ref{app:main-proof} give the full bounds and proofs.

By Lemma~\ref{lem:app-decomposition}, the cumulative regret separates into reward-learning and policy-planning errors:
\begin{align*}
\operatorname{Regret}_{\omega,\alpha,\tau}(K)
\leq{}&
\underbrace{\textstyle\sup_{r\in\mathcal R}\sum_{k=1}^K
\bigl[\ell_k(r_k)-\ell_k(r)\bigr]}_{T_1:\,\text{reward-learning error}}
+
\underbrace{\textstyle\sum_{k=1}^K
\bigl[V_{u_k,1}^{\star,\tau}(s_1)-J_{u_k}^{\tau}(\pi_k)\bigr]}_{T_2:\,\text{policy-planning error}}.
\end{align*}
Then $T_1$ could be further decomposed into the empirical loss on $\hat \ell_k$ and the concentration between the expected loss $\ell_k$ and empirical loss $\hat \ell_k$:
\begin{align*}
T_1= \textstyle
\sup_{r\in\mathcal{R}}\left(\sum_{k=1}^K\bigl[\wh\ell_k(r_k)-\wh\ell_k(r)\bigr]+
\sum_{k=1}^K\bigl[(\ell_k-\wh\ell_k)(r_k)-(\ell_k-\wh\ell_k)(r)\bigr]\right)
\end{align*}
The difficulty in controlling $T_1$ is that the reward iterates depend on both sequential learner feedback and the repeatedly reused expert dataset. Rather than treating the empirical gradient as conditionally unbiased, we retain curvature to control the second term inside the supremum above, which is detailed in the following lemma:

\begin{lemma}[Eluder-controlled stability with retained curvature, informal]
The curvature-matched OMD update ensures, for every $r\in\mathcal R$,
\[ \textstyle 
\sum_{k=1}^K\bigl[\wh\ell_k(r_k)-\wh\ell_k(r)\bigr]
\leq\mathcal O\!\left(H\left[1+\log K+\dim_K(\mathcal R,4/K^2)\right]\right)
-\frac12 \cdot {\alpha(1-\rho)} \mathcal E_K(r),
\]
where $\mathcal E_K(r)$ is the loss curvature enduced by regularization, i.e., 
\[ \textstyle
\mathcal E_K(r):=\sum_{k=1}^K\sum_{h=1}^H
\left[\omega\mathbb E_{\hat d_h^E}[(r_{k,h}-r_h)^2]
+(1-\omega)\mathbb E_{\hat d_{k,h}}[(r_{k,h}-r_h)^2]\right].
\]
This bound holds over the general convex reward class, without requiring the empirical reward seminorm to be nondegenerate.
\end{lemma}

The proximal penalty is $\rho$ times the cumulative empirical loss curvature. The OMD three-point inequality therefore telescopes in this data-dependent seminorm while preserving the $(1-\rho)$ fraction of comparator curvature. Repeated expert observations contribute a harmonic sum, whereas learner updates are controlled by ratios of current squared reward changes to their accumulated magnitude. These ratios yield the generalized eluder dimension of $\mathcal R$ as detailed in Lemmas~\ref{lem:app-proximal} and~\ref{lem:app-reward-stability}.

\begin{lemma}[Variance cancellation under adaptive expert-data reuse, informal]
With probability at least $1-\delta/2$, simultaneously for all $r\in\mathcal R$,
\begin{align*}
&\textstyle{\sum_{k=1}^K}\bigl[(\ell_k-\wh\ell_k)(r_k)-(\ell_k-\wh\ell_k)(r)\bigr]
\leq\tfrac12 \cdot {\alpha(1-\rho)} \mathcal E_K(r)\\
&\qquad+\widetilde{\mathcal O}\!\Big(
(K + N) H / N \cdot \big(\log\!\left(2\overline{\mathcal N}_{\mathcal R}(1/K^2)\right)
+\log\!\left(2\overline{\mathcal N}_{\mathcal R}(1/N^2)\right)\big)
\Big).
\end{align*}
The expert-data bound is uniform over reward pairs, allowing every iterate $r_k$ to depend on the same expert sample.
\end{lemma}

Our concentration argument adds a centered quadratic correction to the sampling noise before applying Freedman's inequality. Its population variance is then canceled by the corresponding population quadratic term, leaving precisely the empirical curvature appearing above. For learner feedback, we apply this argument sequentially and extend uniformly from a $1/K^2$-cover of the comparator class. For expert feedback, we first establish concentration over all pairs in a $1/N^2$-cover, and only then substitute the data-dependent iterate $r_k$ for one member of the pair. This order avoids an independence assumption between the learned reward and the demonstrations; see Lemma~\ref{lem:app-reward-concentration}.

Adding the two inequalities cancels $\mathcal E_K(r)$ exactly, before taking the supremum over $r$. Choosing $\rho=1/2$ balances the $1/\rho$ stability and $1/(1-\rho)$ concentration factors, reserving a constant fraction of curvature for each argument. This gives Lemma~\ref{lem:app-reward-bound}: the online reward cost is controlled by sequential complexity and covering logarithms, while expert reuse contributes $\widetilde{\mathcal O}(KH/N)$. Thus, averaging yields a $1/N$ expert-data contribution without requiring a deterministic expert or fresh demonstrations at each episode.

For $T_2$, we build on the KL-regularized planning analysis of \citet{zhao2025kl}, accounting for the changing shaped rewards and data-dependent continuation values.

\begin{lemma}[Squared-bonus under changing rewards, informal]
With probability at least $1-\delta/2$,
\begin{align*}
T_2
&\leq\frac{2H^2}{\tau}
\sum_{k=1}^K\sum_{h=1}^H
\mathbb E_{d_h^{\pi_k}}\!\left[b_{k,h}^2\right] \leq\mathcal O\!\left(\frac{H^3\beta^2}{\tau}
\left[\dim_K(\mathcal F,\lambda)+\log\frac{8H}{\delta}\right]\right).
\end{align*}
\end{lemma}

Uniform regression confidence ensures optimism despite reuse of historical transitions with recomputed targets. The KL-regularized performance difference identity and Gibbs formulas then make policy error quadratic in the optimistic action-value error. A Bellman recursion reduces this error to squared bonuses.
Finally, concentration transfers the observed squared-bonus bound, controlled by $\dim_K(\mathcal F,\lambda)$, to occupancy expectations. This quadratic dependence, rather than a linear bonus bound, gives the fast planning rate; see Lemma~\ref{lem:app-policy-bound}.

Combining the reward and policy bounds of $T_1, T_2$ presented in aforementioned lemmas, we are ready to control the cumulative dual regret in Eq.~\ref{eqn:regret_bound} by
\begin{align*}
\operatorname{Regret}_{\omega,\alpha,\tau}(K) &\le \tilde {\mathcal O}(H(1 + \log K + \dim_K(\mathcal R,4/K^2))) - \tfrac12 \cdot \alpha (1 - \rho) \mathcal E_K(r) & \text{(Lemma 1)}\\
&\quad + \tfrac12 \cdot \alpha (1 - \rho) \mathcal E_K(r) + \tilde {\mathcal O}((K + N) \cdot H   / N \cdot \log \overline{\mathcal N}_{\mathcal R}) & \text{(Lemma 2)}\\
&\quad + \tilde {\mathcal O}(H^3\beta^2\dim_K(\mathcal F,\lambda) / \tau), & \text{(Lemma 3)}
\end{align*}
which yields the results in Theorem~\ref{thm:main}. Averaging over episodes gives the $\widetilde{\mathcal O}(1/K+1/N)$ regularized saddle-point gap for fixed regularization parameters and controlled function-class complexity.

\section{Conclusion}
In this paper, we considered the provable benefits of reward and policy regularizers in adversarial imitation learning setting.
We proposed Dually Regularized AIL algorithm that minimizes a regularized imitation learning gap.
Our algorithm establishes an $\widetilde{\mathcal{O}}(\frac{1}{K}+\frac{1}{N})$ average regret bound under fixed regularization. This result demonstrates the benefits of regularization where the reward regularization absorbs sampling error while controlling update stability, and policy regularization controls the planner error, which together contributes to a faster convergence rate for AIL. 

\bibliography{iclr2027_conference}
\bibliographystyle{iclr2027_conference}

\appendix
\section{Additional Related Works}
\paragraph{Fast rates in regularized reinforcement learning.}
Maximum-entropy RL provides a practical policy-regularization framework~\citep{haarnoja2018soft}. On the theoretical side, \citet{tiapkin2023fast} establish fast rates for maximum-entropy exploration and entropy-regularized planning, and \citet{tiapkin2024demonstration} study KL-regularized planning toward a reference policy learned from demonstrations. These best-policy-identification guarantees concern the returned policy, rather than cumulative regret during exploration. \citet{zhao2025kl} establish logarithmic regret under general function approximation by combining optimistic value estimation with a squared-Bellman-error decomposition. Our policy analysis builds on this mechanism. Related logarithmic-regret results extend KL regularization to zero-sum Markov games~\citep{nayak2025klgames}. Neither the fixed-reward RL problem nor the Markov-game protocol directly accounts for an adversarial reward learner that repeatedly reuses the same finite expert sample. Controlling this reward-learning error is the additional requirement addressed by our analysis.

\paragraph{Sequential complexity and reward optimization.}
Eluder dimension quantifies the sequential uncertainty of function classes~\citep{NIPS2013_41bfd20a}; generalized versions support optimistic RL with nonlinear function approximation~\citep{agarwal2023vo,zhao2024nearly,zhao2025kl}. We use the unweighted scalar-ridge specialization of generalized eluder dimension for both reward-update stability and value estimation. The generalized eluder coefficient in OPT-AIL is a related but distinct complexity measure~\citep{xu2024optail}. Online mirror descent itself is standard~\citep{hazan2023oco,shani2022online}; our analysis uses the occupancy-weighted reward curvature to control sequential learner noise and finite-expert error while bounding the stability cost directly through the reward class, without assuming a well-conditioned expert feature covariance.
\section{Properties of KL-Regularized RL}
In this section, we present some of the basic properties of KL-regularized RL.

Let $V_{r,h}^{\star,\tau}:=\sup_\pi V_{r,h}^{\pi,\tau}$ and
$Q_{r,h}^{\star,\tau}:=\sup_\pi Q_{r,h}^{\pi,\tau}$ denote the
pointwise optimal values. With $V_{r,H+1}^{\star,\tau}\equiv0$,
they satisfy
\begin{align}
\label{eqn:kl_optimal_q}
    Q_{r,h}^{\star,\tau}(s,a)
    &=r_h(s,a)+[P_hV_{r,h+1}^{\star,\tau}](s,a),\\
\label{eqn:kl_optimal_v}
    V_{r,h}^{\star,\tau}(s)
    &=\tau\log\mathbb{E}_{a\sim\piref_h(\cdot\mid s)}
    \exp\!\left(\frac{Q_{r,h}^{\star,\tau}(s,a)}{\tau}\right).
\end{align}
Existing results on KL-regularized reinforcement
learning~\citep{zhang2023mathematical} yield a closed-form
characterization of the optimal policy, as stated in the following lemma:
\begin{lemma}[Optimal KL-regularized policy, \citet{zhang2023mathematical}]
\label{lem:kl_optimal_policy}
For any bounded stagewise reward $r=\{r_h\}_{h\in[H]}$, the policy
$\pi_r^{\star,\tau}=\{\pi_{r,h}^{\star,\tau}\}_{h\in[H]}$ defined by
\begin{align}
\label{eqn:kl_optimal_policy}
    \pi_{r,h}^{\star,\tau}(a\mid s)
    :=\piref_h(a\mid s)
    \exp\!\left(
    \frac{Q_{r,h}^{\star,\tau}(s,a)-V_{r,h}^{\star,\tau}(s)}{\tau}
    \right)
\end{align}
attains the optimal value $V_{r,h}^{\star,\tau}(s)$ for every
$h\in[H]$ and $s\in\mathcal{S}$.
\end{lemma}

The following identity expresses a policy's regularized value gap
as its expected cumulative KL divergence from the optimal policy.
\begin{lemma}[Performance difference for KL-regularized RL]
\label{lem:kl_performance_difference}
For any bounded stagewise reward $r$ and any policy $\pi$ with
finite expected cumulative KL cost relative to $\piref$, we have
\begin{align}
V_{r,1}^{\star,\tau}(s_1)-J_r^\tau(\pi)
=\tau\mathbb E_{P,\pi}\!\left[
\sum_{h=1}^H\operatorname{KL}\!\left(
\pi_h(\cdot\mid s_h)\,\|\,
\pi_{r,h}^{\star,\tau}(\cdot\mid s_h)\right)\right],
\label{eqn:kl_performance_difference}
\end{align}
where $\pi_r^{\star,\tau}$ is defined in
Lemma~\ref{lem:kl_optimal_policy}.
\end{lemma}
\begin{proof}
On state--action pairs visited by $\pi$, Eq.~\ref{eqn:kl_optimal_policy} gives, almost surely,
\[
\tau\log\frac{\pi_h(a\mid s)}{\pi_{r,h}^{\star,\tau}(a\mid s)}
=\tau\log\frac{\pi_h(a\mid s)}{\piref_h(a\mid s)}
-Q_{r,h}^{\star,\tau}(s,a)+V_{r,h}^{\star,\tau}(s).
\]
Taking expectations along a trajectory under $\pi$ and using
the Bellman equations gives, for each $h\in[H]$,
\begin{align*}
&\mathbb E_{P,\pi}\!\left[
V_{r,h}^{\star,\tau}(s_h)-V_{r,h}^{\pi,\tau}(s_h)\right]\\
&\quad=\tau\mathbb E_{P,\pi}\!\left[
\operatorname{KL}\!\left(
\pi_h(\cdot\mid s_h)\,\|\,
\pi_{r,h}^{\star,\tau}(\cdot\mid s_h)\right)\right]\\
&\qquad+\mathbb E_{P,\pi}\!\left[
V_{r,h+1}^{\star,\tau}(s_{h+1})-V_{r,h+1}^{\pi,\tau}(s_{h+1})\right].
\end{align*}
Summing over $h$ telescopes the value differences. The result follows
from $V_{r,H+1}^{\star,\tau}=V_{r,H+1}^{\pi,\tau}=0$ and
$J_r^\tau(\pi)=V_{r,1}^{\pi,\tau}(s_1)$.
\end{proof}

\section{Proof of Theorem~\ref{thm:main}}
\label{app:main-proof}

\paragraph{Proof conventions.}
The expert has finite expected
KL cost by assumption, and the algorithm's Gibbs policies also have
finite expected KL cost, so the payoff differences below are well defined.

All $\mathcal O(\cdot)$ notation hides only universal numerical constants.
Let $\mathcal H_{k-1}$ be the $\sigma$-algebra generated by the history
before episode $k$, including $\mathcal D^E$ and the choices of
$r_k$ and $\pi_k$.
Let $\mathcal H_{k,h}$ additionally include the observations in
episode $k$ up to $(s_{k,h},a_{k,h})$, before $s_{k,h+1}$ is observed.
Conditional on $\mathcal H_{k-1}$, $(s_{k,h},a_{k,h})$ has distribution
$d_h^{\pi_k}$; conditional on $\mathcal H_{k,h}$, $s_{k,h+1}$ has
distribution $P_h(\cdot\mid s_{k,h},a_{k,h})$.
No independence between stages is assumed.
Within the learner-sampling, regression, and exploration-bonus proofs,
probabilities are conditional on $\mathcal D^E$; their bounds also hold
unconditionally by averaging over $\mathcal D^E$.

\subsection{Main Proof of Theorem~\ref{thm:main}}
\label{app:proof-sketch}

In this subsection, we highlight the key ingredients in the proof of
Theorem~\ref{thm:main}. We use $T_1$ and $T_2$ from
Eq.~\ref{eq:app-decomposition} for the reward-learning and
policy-planning errors, respectively.
Throughout this proof, the assumptions and parameter choices of
Theorem~\ref{thm:main} are in force.

\begin{lemma}[Regret decomposition]
\label{lem:app-decomposition}
For every generated sequence, with
$\bar\pi_K=\operatorname{Unif}\{\pi_1,\ldots,\pi_K\}$ and
$\bar r_K=K^{-1}\sum_{k=1}^K r_k$, we have
\begin{align}
&\operatorname{Regret}_{\omega,\alpha,\tau}(K)
\leq
\underbrace{
\sup_{r\in\mathcal R}\sum_{k=1}^K
\bigl(\ell_k(r_k)-\ell_k(r)\bigr)
}_{T_1:\,\text{reward-learning error}}+
\underbrace{
\sum_{k=1}^K\left[
V_{u_k,1}^{\star,\tau}(s_1)-J_{u_k}^{\tau}(\pi_k)
\right]
}_{T_2:\,\text{policy-planning error}}.
\label{eq:app-decomposition}
\end{align}
Here $\ell_k$ is the loss already defined in
Eq.~\ref{eq:population-reward-loss}.
\end{lemma}
\begin{proof}[Proof of Lemma~\ref{lem:app-decomposition}]
Adding and subtracting $\sum_{k=1}^K\mathcal L(\pi_k,r_k)$ gives
\begin{align}
\operatorname{Regret}_{\omega,\alpha,\tau}(K)
=T_1+\sum_{k=1}^K\mathcal L(\pi_k,r_k)
-\inf_\pi\sum_{k=1}^K\mathcal L(\pi,r_k),
\label{eq:app-direct-decomposition}
\end{align}
where Eq.~\ref{eq:population-reward-loss} gives
$\mathcal L(\pi_k,r)-\mathcal L(\pi_k,r_k)
=\ell_k(r_k)-\ell_k(r)$.

By the definition of 
shaped-reward definition $u_{r, h}$ in Algorithm~\ref{alg:main}
\(
\mathcal L(\pi,r_k)
=J_{u_k}^\tau(\pi^E)-J_{u_k}^\tau(\pi)
-\frac\alpha2\sum_{h=1}^H\mathbb E_{d_h^E}[r_{k,h}^2].
\)
The terms involving the expert are independent of $\pi$, so
\begin{align}
\sum_{k=1}^K\mathcal L(\pi_k,r_k)
-\inf_\pi\sum_{k=1}^K\mathcal L(\pi,r_k)
&\leq\sum_{k=1}^K
\left[\mathcal L(\pi_k,r_k)-\inf_\pi\mathcal L(\pi,r_k)\right]
\nonumber\\
&\quad=\sum_{k=1}^K
\left[V_{u_k,1}^{\star,\tau}(s_1)-J_{u_k}^\tau(\pi_k)\right]
=T_2.
\label{eq:app-current-reward-comparison}
\end{align}
Combining this with Eq.~\ref{eq:app-direct-decomposition} proves
Eq.~\ref{eq:app-decomposition}.
\end{proof}

The first term $T_1$ in Eq.~\ref{eq:app-decomposition} controls reward learning error. 
We control it through the following lemma.

\begin{lemma}[Reward-learning error]
\label{lem:app-reward-bound}
Under the parameter choice $\rho=1/2$ in
Eq.~\ref{eq:reward-update}, with probability at least $1-\delta/2$,
\begin{align*}
T_1
&\leq\mathcal O\!\Bigg(
\frac{H(1+\alpha\omega)^2}{\alpha\omega}(1+\log K)
+\frac{H(1+\alpha(1-\omega))^2}{\alpha(1-\omega)}
\dim_K(\mathcal R,4/K^2)\\
&\qquad+
\frac{H(1+\alpha(1-\omega))^2}{\alpha(1-\omega)}
\log\frac{4H\overline{\mathcal N}_{\mathcal R}(1/K^2)}{\delta}\\
&\qquad+
\frac{KH(1+\alpha\omega)^2}{\alpha\omega N}
\log\frac{4H\overline{\mathcal N}_{\mathcal R}(1/N^2)^2}{\delta}
\Bigg).
\end{align*}
\end{lemma}
\begin{proof}
[Proof Sketch of Lemma~\ref{lem:app-reward-bound}]
Combining the three-point inequality for the reward update with
the exact quadratic expansion of $\wh\ell_k$, Lemma~\ref{lem:app-proximal}
gives
\[
T_1\leq T_S+\sup_{r\in\mathcal R}T_C(r).
\]
Here $T_S$ is the stability term in Eq.~\ref{eq:app-stability-term},
and $T_C(r)$ is the sampling term with retained comparator curvature
in Eq.~\ref{eq:app-sampling-error}.

The quadratic loss curvature offsets growth of the proximal penalty,
leaving squared-update control in $T_S$ and comparator curvature in
$T_C(r)$. Stability is controlled by accumulated reward discrepancies
on observed state--action pairs; Lemma~\ref{lem:app-reward-stability} gives
\begin{align*}
T_S&\leq
\frac{H(1+\alpha\omega)^2}{2\alpha\rho\omega}(1+\log K)
+\frac{H(1+\alpha(1-\omega))^2}{2\alpha\rho(1-\omega)}
\dim_K(\mathcal R,4/K^2)\\
&\qquad+2H(1+\alpha(1-\omega)).
\end{align*}

For sampling error, the retained comparator curvature offsets the
variance terms in Freedman's inequality. Uniform covering arguments
then give, by Lemma~\ref{lem:app-reward-concentration}, with probability
at least $1-\delta/2$,
\begin{align*}
\sup_{r\in\mathcal R}T_C(r)
&\leq\mathcal O\!\Bigg(
\frac{H(1+\alpha(1-\omega))^2}
{\alpha(1-\omega)(1-\rho)}
\log\frac{4H\overline{\mathcal N}_{\mathcal R}(1/K^2)}{\delta}\\
&\qquad+
\frac{KH(1+\alpha\omega)^2}
{\alpha\omega(1-\rho)N}
\log\frac{4H\overline{\mathcal N}_{\mathcal R}(1/N^2)^2}{\delta}
\Bigg).
\end{align*}

Combining the stability and concentration bounds and setting
$\rho=1/2$ proves the claim. The detailed argument is given in
Appendix~\ref{app:reward-bound}.
\end{proof}

\begin{remark}[Choice of the OMD parameter]
\label{rem:omd-rho}
We choose \(\rho=1/2\) for concreteness. The analysis extends to any \(\rho\in(0,1)\), with the stability and sampling-error bounds carrying factors \(1/\rho\) and \(1/(1-\rho)\), respectively. Hence any fixed \(\rho\in(0,1)\) yields the same dependence on the sample budgets, with constants depending on \(\rho\).
\end{remark}

We next control the policy-planning error $T_2$ with the following lemma.

\begin{lemma}[Policy-planning error]
\label{lem:app-policy-bound}
With probability at least $1-\delta/2$,
\begin{align*}
T_2
&\leq\frac{2H^2}{\tau}\sum_{k=1}^K\sum_{h=1}^H
\mathbb E_{d_h^{\pi_k}}[b_{k,h}^2]\\
&\leq\mathcal O\!\left(
\frac{H^3\beta^2}{\tau}
\left[\dim_K(\mathcal F,\lambda)+\log\frac{8H}{\delta}\right]
\right).
\end{align*}
\end{lemma}
\begin{proof}
[Proof Sketch of Lemma~\ref{lem:app-policy-bound}]
The analysis starts from the optimism property
$\wh Q_{k,h}(s,a)\geq Q_{u_k,h}^{\star,\tau}(s,a)$ in
Lemma~\ref{lem:app-optimism}, which holds simultaneously for all
$(k,h,s,a)\in[K]\times[H]\times\mathcal S\times\mathcal A$
with probability at least $1-\delta/4$.

Under optimism, the formulas for the algorithm's policy and the
optimal KL-regularized policy give
\begin{align*}
&\tau\operatorname{KL}\!\left(
\pi_{k,h}(\cdot\mid s)\,\|\,
\pi_{u_k,h}^{\star,\tau}(\cdot\mid s)\right)\\
&\qquad\leq\frac1{2\tau}
\mathbb E_{a\sim\pi_{k,h}(\cdot\mid s)}\!\left[
\bigl(\wh Q_{k,h}(s,a)-Q_{u_k,h}^{\star,\tau}(s,a)\bigr)^2\right].
\end{align*}
The KL-regularized performance difference identity therefore bounds
the policy-planning error quadratically in the action-value
estimation error. Controlling this estimation error through the
exploration bonuses, Lemma~\ref{lem:app-policy-comparison} yields
the squared-bonus bound in the statement.

The generalized eluder dimension of $\mathcal F$ controls the sum
of squared bonuses at the observed state--action pairs.
Lemma~\ref{lem:app-predictable} uses an exponential-moment argument
to transfer this bound to the sum of occupancy expectations
$\sum_{k,h}\mathbb E_{d_h^{\pi_k}}[b_{k,h}^2]$, with failure
probability at most $\delta/4$. Combining this event with the
planning event above by a union bound gives probability at least
$1-\delta/2$ and completes the bound on $T_2$. The detailed argument
is given in Appendix~\ref{app:policy-bound}.
\end{proof}

These two ingredients establish the theorem as follows.

\begin{proof}[Proof of Theorem~\ref{thm:main}]
By Lemma~\ref{lem:app-decomposition},
\[
\operatorname{Regret}_{\omega,\alpha,\tau}(K)
\leq T_1+T_2.
\]
Lemma~\ref{lem:app-reward-bound} bounds $T_1$ with failure
probability at most $\delta/2$, while
Lemma~\ref{lem:app-policy-bound} bounds $T_2$ with failure
probability at most $\delta/2$.
A union bound therefore yields the claimed result with probability
at least $1-\delta$.
Substituting the two bounds gives
Eq.~\ref{eqn:regret_bound}.
\end{proof}

\subsection{Proof of the Reward-Learning Bound}
\label{app:reward-bound}

\subsubsection{Technical Lemmas}
\label{app:reward-lemmas}

In this section, we first state the auxiliary results used to prove
Lemma~\ref{lem:app-reward-bound}.
Proofs of following lemmas are deferred to Appendix~\ref{app:reward-lemma-proofs}.

We start with the following lemma showing that the reward error could be decomposed by OMD analysis into a stability term $T_S$ and a concentration term $T_C$,
which we will bound respectively.
\begin{lemma}[Reward error decomposition]
\label{lem:app-proximal}
Define
\begin{align}
T_S&:=\sum_{k=1}^K\sum_{h=1}^H\left[
\langle\wh g_{k,h},r_{k,h}-r_{k+1,h}\rangle
-\frac{\alpha\rho}{2}\norm{r_{k,h}-r_{k+1,h}}_{k,h}^2\right],
\label{eq:app-stability-term}\\
T_C(r)&:=\sum_{k=1}^K\left[
\ell_k(r_k)-\ell_k(r)-\wh\ell_k(r_k)+\wh\ell_k(r)\right]
\nonumber\\
&\qquad-\frac{\alpha(1-\rho)}{2}
\sum_{k=1}^K\sum_{h=1}^H\Bigl[
\omega\langle \hat d_h^E,(r_{k,h}-r_h)^2\rangle
+(1-\omega)\langle \hat d_{k,h},(r_{k,h}-r_h)^2\rangle
\Bigr].
\label{eq:app-sampling-error}
\end{align}
Then, 
the reward update in Eq.~\ref{eq:reward-update} optimization objective ensures that, for every $r\in\mathcal R$,
\begin{align}
\sum_{k=1}^K\bigl(\ell_k(r_k)-\ell_k(r)\bigr)
&\leq T_S+T_C(r),
&T_1&\leq T_S+\sup_{r\in\mathcal R}T_C(r).
\label{eq:app-curvature-reduction}
\end{align}
\end{lemma}
For the stability term,
we then show with the following lemma that,
it could be bounded by terms logarithmic or constant in $K$ and the Eluder dimension of reward class $\mathcal{R}$.
\begin{lemma}[Reward stability]
\label{lem:app-reward-stability}
With $0<\rho<1$, the reward update in Eq.~\ref{eq:reward-update} guarantees
\begin{align}
T_S&\leq
\frac{H(1+\alpha\omega)^2}{2\alpha\rho\omega}(1+\log K)
+\frac{H(1+\alpha(1-\omega))^2}{2\alpha\rho(1-\omega)}
\dim_K(\mathcal R,4/K^2)\nonumber\\
&\qquad+2H(1+\alpha(1-\omega)).
\label{eq:app-stability-bound}
\end{align}
\end{lemma}
For the concentration term,
we show with the following lemmas that,
it could be bounded by terms logarithmic in $K$ or of order $\tilde{\mathcal{O}}(K/N)$
\begin{lemma}[Freedman's inequality]
\label{lem:app-freedman}
Let $\{X_t\}_{t=1}^n$ be a martingale-difference sequence with respect
to the filtration $\{\mathcal G_t\}_{t=0}^n$, with $|X_t|\leq b$ almost surely for a
deterministic $b>0$, and $v_t:=\mathbb E[X_t^2\mid\mathcal G_{t-1}]$.
For any deterministic $0<\gamma\leq1/b$ and $z>0$,
with probability at least $1-e^{-z}$,
\begin{equation}
\sum_{t=1}^n X_t\leq\gamma\sum_{t=1}^n v_t
+\frac{z}{\gamma}.
\label{eq:app-freedman}
\end{equation}
\end{lemma}
\begin{lemma}
\label{lem:app-reward-concentration}
With $0<\rho<1$, it holds with probability at least $1-\delta/2$ that
\begin{align}
\sup_{r\in\mathcal R}T_C(r)
&\leq\mathcal O\!\Bigg(
\frac{H(1+\alpha(1-\omega))^2}
{\alpha(1-\omega)(1-\rho)}
\log\frac{4H\overline{\mathcal N}_{\mathcal R}(1/K^2)}{\delta}
\nonumber\\
&\qquad+
\frac{KH(1+\alpha\omega)^2}
{\alpha\omega(1-\rho)N}
\log\frac{4H\overline{\mathcal N}_{\mathcal R}(1/N^2)^2}{\delta}
\Bigg).
\end{align}
\end{lemma}

\subsubsection{Proof of Lemma~\ref{lem:app-reward-bound}}
With results established in
Appendix~\ref{app:reward-lemmas},
we provide the proof of Lemma~\ref{lem:app-reward-bound} in the following.
\begin{proof}
First, by Lemmas~\ref{lem:app-proximal},
we separate the reward-learning error $T_1$ into two terms 
\begin{align*}
T_1&\leq T_S+\sup_{r\in\mathcal R}T_C(r).
\end{align*}
Further, with \ref{lem:app-reward-stability} and
\ref{lem:app-reward-concentration},
we bound the two terms respectively
\begin{align*}
T_S&\leq
\frac{H(1+\alpha\omega)^2}{2\alpha\rho\omega}(1+\log K)
+\frac{H(1+\alpha(1-\omega))^2}{2\alpha\rho(1-\omega)}
\dim_K(\mathcal R,4/K^2)\nonumber\\
&\qquad+2H(1+\alpha(1-\omega))\\
\sup_{r\in\mathcal R}T_C(r)
&\leq\mathcal O\!\Bigg(
\frac{H(1+\alpha(1-\omega))^2}
{\alpha(1-\omega)(1-\rho)}
\log\frac{4H\overline{\mathcal N}_{\mathcal R}(1/K^2)}{\delta}
\nonumber\\
&\qquad+
\frac{KH(1+\alpha\omega)^2}
{\alpha\omega(1-\rho)N}
\log\frac{4H\overline{\mathcal N}_{\mathcal R}(1/N^2)^2}{\delta}
\Bigg).
\end{align*}

Finally, setting
$\rho=1/2$. Then
$1/(\alpha\rho)=2/\alpha$ and
$\alpha(1-\rho)=\alpha/2$.
The lower-order stability term is absorbed into the learner
concentration term using
\[
1+\alpha(1-\omega)
\leq
\frac{(1+\alpha(1-\omega))^2}{\alpha(1-\omega)},
\qquad
\log\frac{4H\overline{\mathcal N}_{\mathcal R}(1/K^2)}{\delta}\geq1.
\]
This gives the stated bound up to universal numerical constants.
\end{proof}

\subsubsection{Proofs of the Technical Lemmas}
\label{app:reward-lemma-proofs}

\begin{proof}[Proof of Lemma~\ref{lem:app-proximal}]
The seminorm induces a positive-semidefinite bilinear form
$\langle\cdot,\cdot\rangle_{k,h}$.
The directional derivative of the proximal objective along
$(1-t)r_{k+1}+tr\in\mathcal R$ is nonnegative at $t=0$. Hence
\[
\langle\wh g_k,r_{k+1}-r\rangle
\leq\alpha\rho\sum_{h=1}^H
\langle r_{k+1,h}-r_{k,h},r_h-r_{k+1,h}\rangle_{k,h}.
\]
Rearranging and applying the equality that $2\langle x-y,z-x\rangle=\norm{y-z}^2-\norm{x-z}^2-\norm{x-y}^2$ for $\langle\cdot,\cdot\rangle_{k,h}$ and $\norm{\cdot}_{k,h}$ yield
\begin{align}
&\langle\wh g_k,r_{k+1}-r\rangle
\nonumber\\
&\quad\leq\frac{\alpha\rho}2\sum_{h=1}^H\Bigl(
\norm{r_{k,h}-r_h}_{k,h}^2
-\norm{r_{k+1,h}-r_h}_{k,h}^2
-\norm{r_{k+1,h}-r_{k,h}}_{k,h}^2\Bigr).
\label{eq:app-three-point}
\end{align}

By property of quadratic function, the empirical loss
satisfies
\begin{align}
\sum_{k=1}^K\wh\ell_k(r_k)-\wh\ell_k(r)
={}&\sum_{k=1}^K\langle\wh g_k,r_k-r\rangle
\nonumber\\
&-\frac\alpha2\sum_{k=1}^K\sum_{h=1}^H\Bigl[
\omega\langle \hat d_h^E,(r_{k,h}-r_h)^2\rangle
+(1-\omega)\langle \hat d_{k,h},(r_{k,h}-r_h)^2\rangle
\Bigr].
\label{eq:app-exact-curvature}
\end{align}
Split the first term in RHS through $\langle\wh g_k,r_k-r\rangle=\langle\wh g_k,r_{k+1}-r\rangle+\langle\wh g_k,r_k-r_{k+1}\rangle$ and apply the
three-point inequality,
we then have
\begin{align*}
&\sum_{k=1}^K\langle\wh g_k,r_k-r\rangle\\
={}&\sum_{k=1}^K\langle\wh g_k,r_k-r_{k+1}\rangle+\sum_{k=1}^K\langle\wh g_k,r_{k+1}-r\rangle\\
\le{}&\sum_{k=1}^K\langle\wh g_k,r_k-r_{k+1}\rangle+\frac{\alpha\rho}2\sum_{k=1}^K\sum_{h=1}^H\Bigl(
\norm{r_{k,h}-r_h}_{k,h}^2
-\norm{r_{k+1,h}-r_h}_{k,h}^2
-\norm{r_{k+1,h}-r_{k,h}}_{k,h}^2\Bigr)\\
\le{}&\underbrace{\sum_{k=1}^K\sum_{h=1}^H\left[
\langle\wh g_{k,h},r_{k,h}-r_{k+1,h}\rangle
-\frac{\alpha\rho}{2}\norm{r_{k,h}-r_{k+1,h}}_{k,h}^2\right]}_{T_S}\\
&\qquad+\frac{\alpha\rho}2\sum_{k=1}^K\sum_{h=1}^H\Bigl[
\omega\langle \hat d_h^E,(r_{k,h}-r_h)^2\rangle
+(1-\omega)\langle \hat d_{k,h},(r_{k,h}-r_h)^2\rangle
\Bigr]
\end{align*}
where the last step is by telescoping
and the fact that $\norm{\cdot}_{0,h}=0$,
$\norm{\cdot}_{K,h}\ge0$ and
$\norm{r_{k,h}-r_h}_{k,h}^2-
\norm{r_{k,h}-r_h}_{k-1,h}^2=
\omega\langle \hat d_h^E,(r_{k,h}-r_h)^2\rangle
+(1-\omega)\langle \hat d_{k,h},(r_{k,h}-r_h)^2\rangle$.

Plugging the above inequality back into Eq.~\ref{eq:app-exact-curvature} gives
\begin{align}
&\sum_{k=1}^K\bigl(\wh\ell_k(r_k)-\wh\ell_k(r)\bigr)
\nonumber\\
&\quad\leq T_S-\frac{\alpha(1-\rho)}{2}
\sum_{k=1}^K\sum_{h=1}^H\Bigl[
\omega\langle \hat d_h^E,(r_{k,h}-r_h)^2\rangle
+(1-\omega)\langle \hat d_{k,h},(r_{k,h}-r_h)^2\rangle
\Bigr].
\label{eq:app-empirical-curvature}
\end{align}
Adding the difference $\sum_{k=1}^K\bigl(\ell_k(r_k)-\ell_k(r)\bigr)-\sum_{k=1}^K\bigl(\wh\ell_k(r_k)-\wh\ell_k(r)\bigr)$ proves
Eq.~\ref{eq:app-curvature-reduction}.
\end{proof}
\begin{proof}[Proof of Lemma~\ref{lem:app-reward-stability}]
Fix $k,h$, recall the definition of $\hat{g}_k$ and $\norm{\cdot}_{k,h}$ in Eq.~\ref{eq:observed-gradient} and \ref{eq:reward-seminorm}, we split the corresponding summand of $T_S$ into
two terms. 
\begin{align*}
&\langle\wh g_{k,h},r_{k,h}-r_{k+1,h}\rangle
-\frac{\alpha\rho}{2}\norm{r_{k,h}-r_{k+1,h}}_{k,h}^2\\
={}&\mathbb{E}_{\hat{d}_{k,h}}
\left[(1+{\alpha(1-\omega)}r_{k,h})(r_{k,h}-r_{k+1,h})\right]-
\frac{\alpha(1-\omega)\rho}{2}
\sum_{j=1}^k\mathbb{E}_{\hat{d}_{j,h}}
\left[(r_{k,h}-r_{k+1,h})^2\right]\\
&+\mathbb{E}_{\hat{d}_h^E}
\left[(-1+{\alpha\omega}r_{k,h})(r_{k,h}-r_{k+1,h})\right]-
\frac{\alpha\omega\rho k}{2}
\mathbb{E}_{\hat{d}_h^E}
\left[(r_{k,h}-r_{k+1,h})^2\right]\,.
\end{align*}
Since $|r_{k,h}|\leq1$, the absolute values
of the expert and learner gradient weights are bounded by
$1+\alpha\omega$ and $1+\alpha(1-\omega)$, respectively.

\textbf{For the first term in RHS.} Write
$v_j:=(r_{k,h}(s_{j,h},a_{j,h})
-r_{k+1,h}(s_{j,h},a_{j,h}))^2$, then we have
\begin{align*}
&\mathbb{E}_{\hat{d}_k}
\left[(1+{\alpha(1-\omega)}r_{k,h})(r_{k,h}-r_{k+1,h})\right]-
\frac{\alpha(1-\omega)\rho}{2}
\sum_{j=1}^k\mathbb{E}_{\hat{d}_j}
\left[(r_{k,h}-r_{k+1,h})^2\right]\\
\le{}&(1+\alpha(1-\omega))\sqrt{v_k}-\frac{(1-\omega)\alpha\rho}{2}\sum_{j=1}^kv_j\\
\end{align*}
If $\sum_{j=1}^kv_j=0$, the term vanishes. Otherwise,
\begin{align}
(1+\alpha(1-\omega))\sqrt{v_k}-\frac{(1-\omega)\alpha\rho}{2}\sum_{j=1}^kv_j
&\leq (1+\alpha(1-\omega))\sqrt{v_k},
\nonumber\\
(1+\alpha(1-\omega))\sqrt{v_k}-\frac{(1-\omega)\alpha\rho}{2}\sum_{j=1}^kv_j
&\leq\frac{(1+\alpha(1-\omega))^2}{2\alpha\rho(1-\omega)}
\frac{v_k}{\sum_{j=1}^kv_j}.
\label{eq:app-learner-movement}
\end{align}
the second inequality is by AM-GM.
If $v_k\leq4/K^2$, use the first bound to obtain
$2(1+\alpha(1-\omega))/K$. Otherwise,
\begin{align}
\frac{v_k}{\sum_{j=1}^kv_j}
&\leq\min\left\{1,
\frac{v_k}{4/K^2+\sum_{j<k}v_j}\right\}
\nonumber\\
&\leq\min\left\{1,D_{\mathcal R_h}^2\!\left(
(s_{k,h},a_{k,h});\mathcal D_{k-1,h};4/K^2
\right)\right\}.
\label{eq:app-reward-ratio}
\end{align}
where the last inequality is by the definition of $D_{\mathcal{R}_h}^2$ in Definition~\ref{def:generalized-eluder}.

\textbf{For the second term in RHS.} We have
\begin{align}
&\mathbb{E}_{\hat{d}^E}
\left[(-1+{\alpha\omega}r_{k,h})(r_{k,h}-r_{k+1,h})\right]-
\frac{\alpha\omega\rho k}{2}
\mathbb{E}_{\hat{d}^E}
\left[(r_{k,h}-r_{k+1,h})^2\right]\nonumber\\
\le{}&(1+{\alpha\omega})\sqrt{\mathbb{E}_{\hat{d}^E}
\left[(r_{k,h}-r_{k+1,h})^2\right]}-
\frac{\alpha\omega\rho k}{2}
\mathbb{E}_{\hat{d}^E}
\left[(r_{k,h}-r_{k+1,h})^2\right]\nonumber\\
\leq{}&\frac{(1+\alpha\omega)^2}{2\alpha\omega\rho k}.
\label{eq:app-expert-movement}
\end{align}
where the first inequality is by Cauchy-Schwarz, 
the second inequality is again by AM-GM.

Finally, combining all above, summing over $k,h$ and using
$\sum_{k=1}^Kk^{-1}\leq1+\log K$ and
Definition~\ref{def:generalized-eluder} complete the proof.
\end{proof}
\begin{proof}[Proof of Lemma~\ref{lem:app-freedman}]
Since $e^u\le 1+u+u^2$ for $|u|\le1$, we have
\(
    \mathbb E\bigl[e^{\gamma X_i}\mid\mathcal G_{i-1}\bigr]
    \le
    1+\gamma^2v_i
    \le
    e^{\gamma^2v_i}.
\)
Consequently, the process
\[
    Z_k
    :=
    \exp\left(
        \gamma\sum_{i=1}^k X_i
        -
        \gamma^2\sum_{i=1}^k v_i
    \right),
    \qquad k=0,\ldots,n,
\]
is a nonnegative supermartingale with $Z_0=1$.
Markov's inequality gives
\[
    \mathbb P\left(
        \gamma\sum_{i=1}^n X_i
        -
        \gamma^2\sum_{i=1}^n v_i
        > z
    \right)
    \le e^{-z},
\]
which proves the claim.
\end{proof}
\begin{proof}[Proof of Lemma~\ref{lem:app-reward-concentration}]
Fix a comparator $r\in\mathcal{R}$.
By Eq.~\ref{eq:app-sampling-error}, $T_C(r)$ decomposes into two terms,
namely learner sampling error and expert sampling error as follows.
\begin{align*}
T_C(r)=
&\sum_{k=1}^K\sum_{h=1}^H\Bigg\langle d_h^{\pi_k}-\hat d_{k,h},
(r_{k,h}-r_h)+\frac{\alpha(1-\omega)}2(r_{k,h}^2-r_h^2)\Bigg\rangle\\
&\underbrace{\hspace{125pt}-\sum_{k=1}^K\sum_{h=1}^H\Bigg\langle\hat d_{k,h},\frac{\alpha(1-\omega)(1-\rho)}2(r_{k,h}-r_h)^2\Bigg\rangle}_{\text{learner sampling error}}\\
&+\sum_{k=1}^K\sum_{h=1}^H\Bigg\langle d_h^{E}-\hat d_{h}^E,
-(r_{k,h}-r_h)+\frac{\alpha\omega}2(r_{k,h}^2-r_h^2)\Bigg\rangle\\
&\underbrace{\hspace{125pt}-\sum_{k=1}^K\sum_{h=1}^H\Bigg\langle\hat d_{h}^E,\frac{\alpha\omega(1-\rho)}2(r_{k,h}-r_h)^2\Bigg\rangle}_{\text{expert sampling error}}
\end{align*}

\textbf{Learner sampling error.}
Fix a stage $h$, and define
\begin{align}
\label{eq:app-local-martingale}
Z_{k,h}(r_h):={}&
\Bigg\langle d_h^{\pi_k}-\hat d_{k,h},
(r_{k,h}-r_h)+\frac{\alpha(1-\omega)}2(r_{k,h}^2-r_h^2)\nonumber\\&\hspace{125pt}+\frac{\alpha(1-\omega)(1-\rho)}{2}(r_{k,h}-r_h)^2\Bigg\rangle\,.
\end{align}
Since $r_h$ is fixed and both $r_k$ and $\pi_k$ are $\mathcal{H}_{k-1}$-measurable, it satisfies that $\mathbb E[Z_{k,h}(r_h)\mid\mathcal H_{k-1}]=0$ and 
\begin{align*}
|Z_{k,h}|
\le{}&4(1+{\alpha(1-\omega)})\\
\mathbb{E}[Z_{k,h}(r_h)^2\mid\mathcal{H}_{k-1}]
\le{}&(1+\alpha(1-\omega))^2\mathbb{E}_{d_h^{\pi_k}}[(r_{k,h}-r_h)^2]\,,
\end{align*}
where the inequalities is by the fact that
$0<\rho<1$ and
$|r|\le1, \forall r\in\mathcal{R}$.

Next, 
for each \(h\), fix a \(1/K^2\)-cover
of \(\mathcal R_h\),
whose cardinality is at most 
$\overline{\mathcal N}_{\mathcal R}(1/K^2)$. 
Applying Lemma~\ref{lem:app-freedman} with $X_k=Z_{k,h}$, $\gamma=
\frac{\alpha(1-\omega)(1-\rho)}
{4(1+\alpha(1-\omega))^2}<
\frac1{4(1+\alpha(1-\omega))}$ and
$z=\log\frac{4H\overline{\mathcal N}_{\mathcal R}(1/K^2)}{\delta}$,
and a union bound then give, with probability at least \(1-\delta/4\),
\[
\begin{aligned}
\sum_{k=1}^K Z_{k,h}(r_h)
\le{}&
\frac{\alpha(1-\omega)(1-\rho)}2
\sum_{k=1}^K
\mathbb E_{d_h^{\pi_k}}[(r_{k,h}-r_h)^2]\\&\qquad+\frac{4(1+\alpha(1-\omega))^2}{\alpha(1-\omega)(1-\rho)}
\log\frac{4H\overline{\mathcal N}_{\mathcal R}(1/K^2)}{\delta},
\end{aligned}
\]
simultaneously for all \(h\in[H]\) and
\(r_h\) within the cover.

Rearranging the inequality cancels the variance term in RHS.
Extending from the covers and summing over $h$ therefore yields, for all $r\in\mathcal{R}$,
\begin{align*}
&\sum_{k=1}^K\sum_{h=1}^H\Bigg\langle d_h^{\pi_k}-\hat d_{k,h},
(r_{k,h}-r_h)+\frac{\alpha(1-\omega)}2(r_{k,h}^2-r_h^2)\Bigg\rangle\\
\le{}&\sum_{k=1}^K\sum_{h=1}^H\Bigg\langle\hat d_{k,h},\frac{\alpha(1-\omega)(1-\rho)}2(r_{k,h}-r_h)^2\Bigg\rangle\\
&\qquad+\frac{4H(1+\alpha(1-\omega))^2}{\alpha(1-\omega)(1-\rho)}
\log\frac{4H\overline{\mathcal N}_{\mathcal R}(1/K^2)}{\delta}
+\frac{2H(1+\alpha(1-\omega))}{K}\,.
\end{align*}
The last remainder follows from the covering argument: after rearrangement, the population and empirical reward polynomials are each \((1+\alpha(1-\omega))\)-Lipschitz in \(r_h\), so approximation at scale \(1/K^2\) contributes at most \(2H(1+\alpha(1-\omega))/K\) after summing over \(k,h\).

\textbf{Expert sampling error.}
Similarly, fix $h$ and a pair $q_h,r_h\in\mathcal{R}_h$, define
\begin{align*}
Y_{i,h}(q_h,r_h)
:={}&\Bigg\langle d_h^{E}-\delta_{(s_{h}^i,a_{h}^i)},
-(q_{h}-r_h)+\frac{\alpha\omega}2(q_{h}^2-r_h^2)+\frac{\alpha\omega(1-\rho)}{2}(q_{h}-r_h)^2\Bigg\rangle\,.
\end{align*}
Since $q_h,r_h$ are fixed independently of the observations inside $\mathcal{D}^E$,
we have $\mathbb{E}_{d_h^{E}}[Y_{i,h}(q_h,r_h)]=0$.
Further, by independence of expert trajectories in $\mathcal{D}^E$ and noticing that
\begin{align*}
|Y_{i,h}(q_h,r_h)|\le{}&4(1+\alpha\omega)\\
\mathbb{E}_{d_h^E}[Y_{i,h}(q_h,r_h)^2]\le{}&(1+\alpha\omega)^2\mathbb{E}_{d_h^E}[(q_{h}-r_h)^2]\,,
\end{align*}
applying Lemma~\ref{lem:app-freedman} with 
$X_{i}=Y_{i,h}$, $\gamma=
\frac{\alpha\omega(1-\rho)}
{4(1+\alpha\omega)^2}<
\frac1{4(1+\alpha\omega)}$ and $z=\log\frac{4H\overline{\mathcal N}_{\mathcal R}(1/N^2)^2}{\delta}$,
and a union bound over stages $h$ and pairs of cover elements give,
with probability at least \(1-\delta/4\),
\begin{align}
\sum_{i=1}^N Y_{i,h}(q_h,r_h)
\le{}&\frac{\alpha\omega(1-\rho)}2
N
\mathbb E_{d_h^{E}}[(q_{h}-r_h)^2]+\frac{4(1+\alpha\omega)^2}{\alpha\omega(1-\rho)}
\log\frac{4H\overline{\mathcal N}_{\mathcal R}(1/N^2)^2}{\delta},
\end{align}
simultaneously for all \(h\in[H]\) and
\(q_h, r_h\) within the cover.

Again, rearranging terms cancels the quadratic term in RHS.
Then, dividing both sides by $N$, extending from the covers, plugging in rewards $q_h=r_{k,h}$ while retaining the same fixed comparator $r_h$, and summing up over $k\in[K]$, $h\in[H]$ yield, 
for all $r\in\mathcal{R}$
\begin{align*}
&\sum_{k=1}^K\sum_{h=1}^H\Bigg\langle d_h^{E}-\hat d_{h}^E,
-(r_{k,h}-r_h)+\frac{\alpha\omega}2(r_{k,h}^2-r_h^2)\Bigg\rangle\\
\le{}&\sum_{k=1}^K\sum_{h=1}^H\Bigg\langle\hat d_{h}^E,\frac{\alpha\omega(1-\rho)}2(r_{k,h}-r_h)^2\Bigg\rangle\\
&\qquad+
\frac{4(1+\alpha\omega)^2HK}{\alpha\omega(1-\rho)N}
\log\frac{4H\overline{\mathcal N}_{\mathcal R}(1/N^2)^2}{\delta}+\frac{8(1+\alpha\omega)HK}{N^2}
\end{align*}
The last remainder follows by approximating both rewards $q_h$ and $r_h$ within \(1/N^2\) in sup norm. For each $h$ and $k$, this changes the reward polynomials by at most \(8(1+\alpha\omega)/N^2\).

\textbf{Combining the bounds.}
Finally, intersecting the events on which the learner and expert sampling bounds hold and applying an union bound gives, with probability at least \(1-\delta/2\), for all \(r\in\mathcal R\),
\[
\begin{aligned}
T_C(r)\le \mathcal O\!\Bigg(
&\frac{H(1+\alpha(1-\omega))^2}
{\alpha(1-\omega)(1-\rho)}
\log\frac{4H\overline{\mathcal N}_{\mathcal R}(1/K^2)}{\delta}\\
&+\frac{KH(1+\alpha\omega)^2}
{\alpha\omega(1-\rho)N}
\log\frac{4H\overline{\mathcal N}_{\mathcal R}(1/N^2)^2}{\delta}
\Bigg),
\end{aligned}
\]
where the cover-extension remainders are absorbed using \(K,N\ge1\). Taking the supremum over \(r\in\mathcal R\) completes the proof.

\end{proof}

\subsection{Proof of the Policy-Planning Error Bound}
\label{app:policy-bound}
\subsubsection{Technical Lemmas}
\label{app:policy-lemmas}

In this section, we state the auxiliary results used to prove
Lemma~\ref{lem:app-policy-bound}.

We start by defining the possible action- and state-value function classes as follows.

\begin{definition}
\label{def:planner-value-class}
For the regression classes $\mathcal F_h$, the bonus classes in
Definition~\ref{def:bonus-class}, and the fixed reference policy
$\piref$, define
\begin{align}
\mathcal Q_h
&:=\left\{\clip_{[-V_{\max},V_{\max}]}(f+b):
f\in\mathcal F_h,\ b\in\mathcal B_h\right\},
\label{eq:planner-q-class}\\
\mathcal V_h
&:=\left\{s\mapsto\tau\log
\mathbb E_{a\sim\piref_h(\cdot\mid s)}
\exp\!\left(\tau^{-1}{Q(s,a)}\right):Q\in\mathcal Q_h\right\}.
\label{eq:planner-value-class}
\end{align}
Set $\mathcal V_{H+1}:=\{0\}$. Every $Q\in\mathcal Q_h$ and
$V\in\mathcal V_h$ satisfies
$\norm{Q}_\infty,\norm{V}_\infty\leq V_{\max}$.
These classes describe the algorithm's optimistic estimates.
\end{definition}

On every possible data history, the regression update gives
$\wh f_{k,h}\in\mathcal F_h$ and Definition~\ref{def:bonus-class}
gives $b_{k,h}\in\mathcal B_h$. Consequently, the clipping and
soft-value updates give $\wh Q_{k,h}\in\mathcal Q_h$ and
$\wh V_{k,h}\in\mathcal V_h$, with $\wh V_{k,H+1}\equiv0$.
This membership holds without conditioning on a confidence event.

We then show that the the covering numbers of these two classes could be bounded by those of $\mathcal{F}$ and $\mathcal{B}$.
\begin{lemma}
\label{lem:app-planner-envelope}
For the classes in Definition~\ref{def:planner-value-class},
it satisfies that, for all $h\in[H]$ and $\kappa>0$
\begin{equation}
\mathcal N_{\mathcal V_h}(\kappa)
\leq\mathcal N_{\mathcal Q_h}(\kappa)
\leq\mathcal N_{\mathcal F_h}(\kappa/2)
\mathcal N_{\mathcal B_h}(\kappa/2).
\label{eq:app-planner-value-cover}
\end{equation}
And, for step $H+1$, $\mathcal N_{\mathcal V_{H+1}}(\kappa)=1$.
\end{lemma}
With the bound for the covering number of $\mathcal{V}_h$ in hand,
we then use the following lemma to establish a pointwise regression-error bound.
\begin{lemma}
\label{lem:app-regression-confidence}
Under Assumption~\ref{ass:optimistic_planner}, choose $\beta$
as in Theorem~\ref{thm:main}, 
with probability at least $1-\delta/4$,
it holds that for all
$k\in[K]$, $h\in[H]$, and $(s,a)\in\mathcal S\times\mathcal A$,
\begin{align}
&\left|\wh f_{k,h}(s,a)-u_{k,h}(s,a)
-[P_h\wh V_{k,h+1}](s,a)\right|
\nonumber\\
&\qquad\leq
\min\left\{4V_{\max},\,
\beta D_{\mathcal F_h}\!\left(
(s,a);\mathcal D_{k-1,h};\lambda\right)\right\}
=b_{k,h}(s,a).
\label{eqn:planner_confidence}
\end{align}
\end{lemma}
Based on the above lemma,
we then show with the following lemma that,
with high probability,
our action-value estimator $\hat{Q}_{k,h}$ will serve as a point-wise optimistic upper bound of the optimal action-value function $Q_{u_k,h}^{\star,\tau}$.
\begin{lemma}[Optimism]
\label{lem:app-optimism}
With probability at least $1-\delta/4$,
for all $k,h,s,a$,
\begin{align}
\wh Q_{k,h}(s,a)&\geq Q_{u_k,h}^{\star,\tau}(s,a),
&\wh V_{k,h}(s)&\geq V_{u_k,h}^{\star,\tau}(s),
\label{eq:app-optimism}\\
\wh Q_{k,h}(s,a)-u_{k,h}(s,a)
-[P_h\wh V_{k,h+1}](s,a)&\leq2b_{k,h}(s,a).
\label{eq:app-residual-upper-bound}
\end{align}
The residual on the left of Eq.~\ref{eq:app-residual-upper-bound}
need not be nonnegative under global clipping.
\end{lemma}
On the confidence event of the above lemma,
we follow Lemma B.1 from \citet{zhao2025kl} to derive the following bounds on the policy-planning error.
\begin{lemma}
\label{lem:app-policy-comparison}
With probability at least $1-\delta/4$,
\begin{equation}
T_2\leq\frac{2H^2}{\tau}\sum_{k=1}^K\sum_{h=1}^H
\mathbb E_{d_h^{\pi_k}}[b_{k,h}^2].
\label{eq:app-policy-comparison}
\end{equation}
\end{lemma}

\begin{lemma}
\label{lem:app-predictable}
For $\beta\geq4V_{\max}$, with probability at least $1-\delta/4$,
\begin{equation}
\sum_{k=1}^K\sum_{h=1}^H
\mathbb E_{d_h^{\pi_k}}[b_{k,h}^2]
\leq2H\beta^2
\left[\dim_K(\mathcal F,\lambda)+\log\frac{4H}{\delta}\right].
\label{eq:app-bonus-sum}
\end{equation}
\end{lemma}


\subsubsection{Proof of Lemma~\ref{lem:app-policy-bound}}
With results established in Appendix~\ref{app:policy-lemmas},
we provide the proof of Lemma~\ref{lem:app-policy-bound} in the following.

\begin{proof}[Proof of Lemma~\ref{lem:app-policy-bound}]
By Lemma~\ref{lem:app-regression-confidence}, the confidence event
holds with probability at least $1-\delta/4$.
On this event, Lemma~\ref{lem:app-optimism} establishes optimism
and the Bellman residual bound, so
Lemma~\ref{lem:app-policy-comparison} yields the first inequality below.

The choice in $\beta$ in Theorem~\ref{thm:main}
satisfies
$\beta^2\geq64V_{\max}^2\log8>16V_{\max}^2$, so
Lemma~\ref{lem:app-predictable} applies with failure probability
at most $\delta/4$. By a union bound, its conclusion and the
confidence event hold simultaneously with probability at least
$1-\delta/2$. On their intersection,
\begin{align*}
T_2
&\leq\frac{2H^2}{\tau}
\sum_{k=1}^K\sum_{h=1}^H
\mathbb E_{d_h^{\pi_k}}[b_{k,h}^2]\\
&\leq\frac{4H^3\beta^2}{\tau}
\left[\dim_K(\mathcal F,\lambda)+\log\frac{4H}{\delta}\right].
\end{align*}
Since $\log(4H/\delta)\leq\log(8H/\delta)$, this proves the
stated bound.
\end{proof}

\subsubsection{Proofs of the technical lemmas}
\label{app:lemma-proofs}

\begin{proof}[Proof of Lemma~\ref{lem:app-planner-envelope}]
Consider two bounded action-value functions
$Q,Q'$ with $\norm{Q-Q'}_\infty\leq\varepsilon$. The pointwise
inequality $Q\leq Q'+\varepsilon$ implies
\[
\tau\log\mathbb E_{a\sim\piref_h(\cdot\mid s)}\exp(\tau^{-1}{Q(s,a)})
\leq
\tau\log\mathbb E_{a\sim\piref_h(\cdot\mid s)}\exp(\tau^{-1}{Q'(s,a)})
+\varepsilon.
\]
Exchanging $Q$ and $Q'$ shows that the soft-value map is
$1$-Lipschitz in the sup norm. Applying this map to a
$\kappa$-cover of $\mathcal Q_h$ gives a $\kappa$-cover of
$\mathcal V_h$, proving the first inequality in
Eq.~\ref{eq:app-planner-value-cover}.
Clipping is also $1$-Lipschitz. Thus, for
$f,f'\in\mathcal F_h$ and $b,b'\in\mathcal B_h$, the corresponding
members $Q,Q'\in\mathcal Q_h$ satisfy
\[
\norm{Q-Q'}_\infty
\leq\norm{f-f'}_\infty+\norm{b-b'}_\infty.
\]
Taking a $\kappa/2$-cover of each input class and applying the
clipped-sum map in Eq.~\ref{eq:planner-q-class} gives a
$\kappa$-cover of $\mathcal Q_h$ whose centers lie in $\mathcal Q_h$.
Counting these pairs proves the second inequality.
The terminal case follows from $\mathcal V_{H+1}=\{0\}$.
\end{proof}

\begin{proof}[Proof of Lemma~\ref{lem:app-regression-confidence}]
Fix $h\in[H]$, $f,g\in\mathcal F_h$, and
$V\in\mathcal V_{h+1}$. Only in this proof, write
\begin{align*}
\Delta_j&:=f(s_{j,h},a_{j,h})-g(s_{j,h},a_{j,h}),\\
\xi_j&:=V(s_{j,h+1})-[P_hV](s_{j,h},a_{j,h}).
\end{align*}
Then $\Delta_j$ is $\mathcal H_{j,h}$-measurable and
$\mathbb E[\xi_j\mid\mathcal H_{j,h}]=0$.
Since $\norm{V}_\infty\leq V_{\max}$, conditional Hoeffding's
lemma gives, for every $t\in\mathbb R$,
\begin{equation}
\mathbb E[\exp(t\xi_j)\mid\mathcal H_{j,h}]
\leq\exp(t^2V_{\max}^2/2).
\end{equation}
With the fixed parameter $\theta:=1/(8V_{\max}^2)$, it follows that
\begin{equation}
\mathbb E\!\left[
\exp\!\left(\theta\left[2\Delta_j\xi_j-\frac14\Delta_j^2\right]\right)
\,\middle|\,\mathcal H_{j,h}\right]\leq1.
\end{equation}
Iterated conditional expectation across episodes and Markov's
inequality therefore imply, for each fixed $k$ and triple $(f,g,V)$,
with probability at least $1-\varepsilon$,
\begin{equation}
2\sum_{j<k}\Delta_j\xi_j-\frac14\sum_{j<k}\Delta_j^2
\leq8V_{\max}^2\log\frac1\varepsilon.
\label{eq:app-regression-fixed-triple}
\end{equation}

Take covers at radius $\kappa:=V_{\max}/K^2$. By
Lemma~\ref{lem:app-planner-envelope}, at each stage the number
of triples of cover elements is at most
\begin{equation}
\begin{aligned}
&\mathcal N_{\mathcal F_h}(\kappa)^2
\mathcal N_{\mathcal V_{h+1}}(\kappa)\\
&\quad\leq\mathcal N_{\mathcal F_h}(\kappa)^2
\mathcal N_{\mathcal F_{h+1}}(\kappa/2)
\mathcal N_{\mathcal B_{h+1}}(\kappa/2)\\
&\quad\leq\overline{\mathcal N}_{\mathcal F\oplus\mathcal B}(\kappa)^2.
\end{aligned}
\end{equation}
Apply Eq.~\ref{eq:app-regression-fixed-triple} to every triple,
stage, and $k\in[K]$, with
$\varepsilon:=\delta/[4HK\overline{\mathcal N}_{\mathcal F\oplus\mathcal B}(\kappa)^2]$.
The resulting union bound has probability at least
$1-\delta/4$.
To extend it beyond the covers, choose approximants
$f',g',V'$ within $\kappa$ in sup norm and let $\Delta_j',\xi_j'$ denote
the corresponding expressions. The structural bounds imply
\begin{align*}
|\Delta_j|,|\Delta_j'|&\leq4V_{\max},
&|\xi_j|,|\xi_j'|&\leq2V_{\max},\\
|\Delta_j-\Delta_j'|&\leq2\kappa,
&|\xi_j-\xi_j'|&\leq2\kappa.
\end{align*}
Consequently,
\begin{equation}
\left|2\Delta_j\xi_j-\frac14\Delta_j^2
-2\Delta_j'\xi_j'+\frac14(\Delta_j')^2\right|
\leq28V_{\max}\kappa.
\end{equation}
Using $k-1\leq K$ and
$\log(4HK\overline{\mathcal N}_{\mathcal F\oplus\mathcal B}(\kappa)^2/\delta)
\leq2\log(8HK\overline{\mathcal N}_{\mathcal F\oplus\mathcal B}(\kappa)/\delta)$,
we obtain, simultaneously for all $f,g,V,k,h$,
\begin{align}
2\sum_{j<k}\Delta_j\xi_j
&\leq\frac14\sum_{j<k}\Delta_j^2
+16V_{\max}^2\log\frac{8HK\overline{\mathcal N}_{\mathcal F\oplus\mathcal B}(\kappa)}{\delta}
+\frac{28V_{\max}^2}{K}.
\label{eq:app-uniform-regression-noise}
\end{align}

Now select the data-dependent functions
$f=\wh f_{k,h}$, $g=u_{k,h}+P_h\wh V_{k,h+1}$, and
$V=\wh V_{k,h+1}$. The regression update gives $f\in\mathcal F_h$,
Bellman completeness gives $g\in\mathcal F_h$, and
the observation following Definition~\ref{def:planner-value-class}
gives $V\in\mathcal V_{h+1}$.
Thus all three functions are covered by the uniform event. The regression targets are
$g(s_{j,h},a_{j,h})+\xi_j$. Comparing the exact least-squares
minimizer with $g$ gives
\begin{equation}
\sum_{j<k}(\Delta_j-\xi_j)^2\leq\sum_{j<k}\xi_j^2,
\qquad\text{hence}\qquad
\sum_{j<k}\Delta_j^2\leq2\sum_{j<k}\Delta_j\xi_j.
\end{equation}
Combining this with Eq.~\ref{eq:app-uniform-regression-noise} yields
\begin{align}
\sum_{j<k}\Delta_j^2
&\leq\frac{64}{3}V_{\max}^2
\log\frac{8HK\overline{\mathcal N}_{\mathcal F\oplus\mathcal B}(\kappa)}{\delta}
+\frac{112V_{\max}^2}{3K}
\nonumber\\
&\leq64V_{\max}^2
\log\frac{8HK\overline{\mathcal N}_{\mathcal F\oplus\mathcal B}(\kappa)}{\delta},
\label{eq:app-regression-empirical-error}
\end{align}
where the last step uses $K\geq1$ and a logarithm of at least
$\log8$. For $k=1$, the empirical sum is empty and this bound
holds directly. Uniformity is essential: the current reward and
continuation value may depend on the same historical transitions
used for fitting; they are selected only after the uniform event
has been established.

By Definition~\ref{def:generalized-eluder}, the pair
$\wh f_{k,h},g\in\mathcal F_h$ satisfies, for every $(s,a)$,
\begin{align*}
|\wh f_{k,h}(s,a)-g(s,a)|
&\leq D_{\mathcal F_h}\!\left(
(s,a);\mathcal D_{k-1,h};\lambda\right)
\sqrt{\lambda+\sum_{j<k}\Delta_j^2}\\
&\leq\beta D_{\mathcal F_h}\!\left(
(s,a);\mathcal D_{k-1,h};\lambda\right).
\end{align*}
Both functions have sup norm at most $2V_{\max}$, so their
pointwise difference is also at most $4V_{\max}$. Taking the
minimum of these two bounds proves Eq.~\ref{eqn:planner_confidence}.
The radius in $\beta$ defined in Theorem~\ref{thm:main} 
moreover satisfies
$\beta^2\geq64V_{\max}^2\log8>16V_{\max}^2$, as required by
the bonus cap.
\end{proof}

\begin{proof}[Proof of Lemma~\ref{lem:app-optimism}]
Conditioned on the event defined in
Lemma~\ref{lem:app-regression-confidence},
which holds with probability at least $1-\delta/4$,
we have
\begin{equation}
u_{k,h}+P_h\wh V_{k,h+1}
\leq\wh f_{k,h}+b_{k,h}
\leq u_{k,h}+P_h\wh V_{k,h+1}+2b_{k,h}.
\label{eq:app-unclipped-confidence}
\end{equation}
The bound $\norm{u_{k,h}}_\infty\leq V_{\max}/H$ and the
soft Bellman recursion give
\begin{equation}
\begin{aligned}
\left|Q_{u_k,h}^{\star,\tau}(s,a)\right|&\leq(H-h+1)V_{\max}/H,\\
\left|V_{u_k,h}^{\star,\tau}(s)\right|&\leq(H-h+1)V_{\max}/H.
\end{aligned}
\label{eq:app-optimal-value-range}
\end{equation}
The value lower bound follows by choosing $\piref$, whose KL
cost is zero; the $Q$ lower bound follows from its Bellman equation.
The upper bounds use nonnegativity of KL.

Proceed by backward induction from stage $H+1$.
If $\wh V_{k,h+1}\geq V_{u_k,h+1}^{\star,\tau}$, then
$u_{k,h}+P_h\wh V_{k,h+1}\geq Q_{u_k,h}^{\star,\tau}$.
The unclipped estimate is therefore optimistic by
Eq.~\ref{eq:app-unclipped-confidence}. The clipping interval contains
the optimal value range, so clipping preserves this inequality.
Monotonicity of the soft Bellman map gives optimism for
$\wh V_{k,h}$ as well.

The same induction gives
$u_{k,h}+P_h\wh V_{k,h+1}\geq Q_{u_k,h}^{\star,\tau}\geq-V_{\max}$,
so lower clipping is inactive. Upper clipping can only reduce the
estimate, and Eq.~\ref{eq:app-unclipped-confidence} yields
Eq.~\ref{eq:app-residual-upper-bound}.
The backup $u_{k,h}+P_h\wh V_{k,h+1}$ may exceed $V_{\max}$,
so the proof does not assert that the clipped residual is nonnegative.
\end{proof}

\begin{proof}
[Proof of Lemma~\ref{lem:app-policy-comparison}]
We condition the proof on the event defined in
Lemma~\ref{lem:app-optimism},
which holds with probability at least $1-\delta/4$.
Recall from Eq.~\ref{eq:app-decomposition} that
\begin{align*}
T_2=\sum_{k=1}^K\left[
V_{u_k,1}^{\star,\tau}(s_1)-J_{u_k}^{\tau}(\pi_k)
\right]
\end{align*}
Fix $k$. Applying Lemma~\ref{lem:kl_performance_difference}
with $r=u_k$ and $\pi=\pi_k$ gives
\begin{align*}
V_{u_k,1}^{\star,\tau}(s_1)-J_{u_k}^\tau(\pi_k)
=\tau\mathbb E_{P,\pi_k}\!\left[
\sum_{h=1}^H\operatorname{KL}\!\left(
\pi_{k,h}(\cdot\mid s_h)\,\|\,
\pi_{u_k,h}^{\star,\tau}(\cdot\mid s_h)\right)\right].
\end{align*}

Further, fix $(k,h,s)$, condition on the optimism event in Lemma~\ref{lem:app-optimism} and by the definition of $\pi_{k,h}$ and $\pi_{u_k,h}^{\star,\tau}$ we know
\begin{align}
&\tau\operatorname{KL}\!\left(
\pi_{k,h}(\cdot\mid s)\,\|\,
\pi_{u_k,h}^{\star,\tau}(\cdot\mid s)\right)
\nonumber\\
&\quad=\mathbb E_{a\sim\pi_{k,h}(\cdot\mid s)}\!\left[
\wh Q_{k,h}(s,a)-Q_{u_k,h}^{\star,\tau}(s,a)\right]
\nonumber\\
&\qquad+\tau\log\mathbb E_{a\sim\pi_{k,h}(\cdot\mid s)}\!\left[
\exp\!\left(-\frac{\wh Q_{k,h}(s,a)-Q_{u_k,h}^{\star,\tau}(s,a)}{\tau}\right)
\right]
\nonumber\\
&\quad\leq\frac1{2\tau}
\mathbb E_{a\sim\pi_{k,h}(\cdot\mid s)}\!\left[
\bigl(\wh Q_{k,h}(s,a)-Q_{u_k,h}^{\star,\tau}(s,a)\bigr)^2\right].
\label{eq:app-one-step-quadratic}
\end{align}
where the inequality is by the optimism event $\widehat Q_{k,h}(s,a)-Q_{u_k,h}^{\star,\tau}(s,a)\ge0$, and the fact that $e^{-x}\leq1-x+x^2/2$ for $x\geq0$,
followed by $\log y\leq y-1$.

For the error term $\widehat Q_{k,h}(s,a)-Q_{u_k,h}^{\star,\tau}(s,a)$,
at $h=H$, optimism, the zero terminal continuation values, and
Eq.~\ref{eq:app-residual-upper-bound} give
\[
0\leq\wh Q_{k,H}(s,a)-Q_{u_k,H}^{\star,\tau}(s,a)
\leq2b_{k,H}(s,a).
\]
For $h<H$, we have
\begin{align}
0\leq{}&\wh Q_{k,h}(s,a)-Q_{u_k,h}^{\star,\tau}(s,a)\nonumber\\
\leq{}&2b_{k,h}(s,a)+\left[P_h\!\left(
\wh V_{k,h+1}-V_{u_k,h+1}^{\star,\tau}\right)\right](s,a)\nonumber\\
\leq{}&2b_{k,h}(s,a)+\mathbb E_{P,\pi_k}\!\left[
\wh Q_{k,h+1}(s_{h+1},a_{h+1})-Q_{u_k,h+1}^{\star,\tau}(s_{h+1},a_{h+1})
\,\middle|\,s_h=s,a_h=a
\right]\nonumber\\
\leq{}&2\mathbb E_{P,\pi_k}\!\left[
\sum_{\ell=h}^H b_{k,\ell}(s_{\ell},a_{\ell})
\,\middle|\,s_h=s,a_h=a\right].
\label{eq:app-error-recursion}
\end{align}
where the first step is the optimism event,
the second step is by the optimal Bellman equation and
Eq.~\ref{eq:app-residual-upper-bound}; the third step follows from
the formulas for the two policies and nonnegativity of KL divergence.
The last step follows by backward induction from the terminal bound.
Together with the case $h=H$, the final bound holds for every $h\in[H]$.

Finally, combining all above, we have
\begin{align}
&V_{u_k,1}^{\star,\tau}(s_1)-J_{u_k}^\tau(\pi_k)
\nonumber\\
&\quad=\tau\mathbb E_{P,\pi_k}\!\left[
\sum_{h=1}^H\operatorname{KL}\!\left(
\pi_{k,h}(\cdot\mid s_h)\,\|\,
\pi_{u_k,h}^{\star,\tau}(\cdot\mid s_h)\right)\right]
\nonumber\\
&\quad\leq\frac2\tau\mathbb E_{P,\pi_k}\!\left[
\sum_{h=1}^H\left(\sum_{\ell=h}^H b_{k,\ell}(s_{\ell},a_{\ell})\right)^2\right]
\nonumber\\
&\quad\leq\frac{2H^2}{\tau}\sum_{h=1}^H
\mathbb E_{d_h^{\pi_k}}[b_{k,h}^2].
\label{eq:app-episode-policy-bound}
\end{align}
where the first inequality combines
Eqs.~\ref{eq:app-one-step-quadratic}
and~\ref{eq:app-error-recursion} with conditional Jensen's inequality.
The last inequality follows from Cauchy--Schwarz and summing over $h$.
Summing over $k$ gives the claimed bound on $T_2$.
    
\end{proof}
\begin{proof}[Proof of Lemma~\ref{lem:app-predictable}]
Fix $h$ and, only in this proof, let
$X_{k,h}:=b_{k,h}(s_{k,h},a_{k,h})^2/\beta^2$.
The bonus cap and $\beta\geq4V_{\max}$ imply $X_{k,h}\in[0,1]$.
Since $e^{-x}\leq1-(1-e^{-1})x$ on $[0,1]$,
\begin{equation}
\mathbb E[e^{-X_{k,h}}\mid\mathcal H_{k-1}]
\leq\exp\!\left(-(1-e^{-1})
\mathbb E[X_{k,h}\mid\mathcal H_{k-1}]\right).
\end{equation}
Iterated conditional expectation and Markov's inequality, followed
by a union bound over stages, give with probability at least
$1-\delta/4$,
\begin{align}
\sum_{k=1}^K\mathbb E[X_{k,h}\mid\mathcal H_{k-1}]
&\leq\frac1{1-e^{-1}}\left[
\sum_{k=1}^K X_{k,h}+\log\frac{4H}{\delta}\right]
\nonumber\\
&\leq2\sum_{k=1}^K X_{k,h}+2\log\frac{4H}{\delta}
\label{eq:app-predictable-conversion}
\end{align}
simultaneously for all $h$.
By the bonus definition,
\begin{equation}
\frac{b_{k,h}(s,a)^2}{\beta^2}
\leq\min\!\left\{1,D_{\mathcal F_h}^2\!\left(
(s,a);\mathcal D_{k-1,h};\lambda\right)\right\}.
\end{equation}
For every $(s,a)$, $b_{k,h}(s,a)$ is $\mathcal H_{k-1}$-measurable, so
$\mathbb E[X_{k,h}\mid\mathcal H_{k-1}]
=\mathbb E_{d_h^{\pi_k}}[b_{k,h}^2]/\beta^2$.
Sum Eq.~\ref{eq:app-predictable-conversion} over $h$, use
Definition~\ref{def:generalized-eluder}, and multiply by $\beta^2$
to conclude.
\end{proof}

\section{Proof of Corollary~\ref{cor:sample_complexity}}
\label{app:sample-complexity-proof}

\begin{proof}
[Proof of Corollary~\ref{cor:sample_complexity}]
Work on the event of Theorem~\ref{thm:main}.
Convexity of $\mathcal R$ gives $\bar r_K\in\mathcal R$.
Since $\mathcal L(\pi,\cdot)$ is concave,
$\mathcal L(\pi,\bar r_K)\geq K^{-1}\sum_{k=1}^K\mathcal L(\pi,r_k)$.
The dual gap is nonnegative because
\[
\inf_\pi\mathcal L(\pi,\bar r_K)
\leq\mathcal L(\bar\pi_K,\bar r_K)
\leq\sup_{r\in\mathcal R}\mathcal L(\bar\pi_K,r).
\]
Using the definition of the episode-level mixture, we therefore obtain
\begin{align}
0\leq\sup_{r \in \mathcal R} \mathcal L(\bar \pi, r) - \inf_{\pi} \mathcal L(\pi, \bar r)
&\leq\frac1K\left[
\sup_{r\in\mathcal R}\sum_{k=1}^K\mathcal L(\pi_k,r)
-\inf_\pi\sum_{k=1}^K\mathcal L(\pi,r_k)\right]
\nonumber\\
&=\frac{\operatorname{Regret}_{\omega,\alpha,\tau}(K)}K.
\label{eq:app-gap-regret}
\end{align}
Since $\lambda\leq V_{\max}^2$ 
for the $\beta$ defined in Theorem~\ref{thm:main}, we can verify that 
\[
\beta^2\leq65V_{\max}^2
\log\frac{8HK\overline{\mathcal N}_{\mathcal F\oplus\mathcal B}(V_{\max}/K^2)}{\delta}.
\]
Divide Eq.~\ref{eqn:regret_bound} by $K$ and choose the integer
sample counts so that the online and expert contributions are each
at most $\epsilon/2$. This gives the sufficient orders displayed in
Corollary~\ref{cor:sample_complexity}.
Precisely, its asymptotic sample requirements mean choosing counts
at least sufficiently large universal multiples of the displayed
expressions, including the logarithmic factors suppressed by
$\widetilde{\mathcal O}$; an arbitrary upper bound on the counts
would not imply sufficiency. Because the complexities depend on
$K$ and $N$, these inequalities must hold jointly at the selected
counts. No small-complexity growth rate follows from boundedness
alone.
\end{proof}

\end{document}

%% file: math_commands.tex
\usepackage{amsmath,amsfonts,bm}

\def\eqref#1{equation~\ref{#1}}

\def\1{\bm{1}}

\DeclareMathAlphabet{\mathsfit}{\encodingdefault}{\sfdefault}{m}{sl}
\SetMathAlphabet{\mathsfit}{bold}{\encodingdefault}{\sfdefault}{bx}{n}

